\documentclass[11pt, a4paper, oneside, reqno]{amsart}

\usepackage[dvipsnames]{xcolor}
\definecolor{darkblue}{rgb}{0.0, 0.0, 0.45}
\definecolor{lightblue}{RGB}{240,248,255}
\definecolor{lightblue2}{rgb}{0.68, 0.85, 0.9}
\definecolor{lightcyan}{rgb}{0.88, 1.0, 1.0}
\definecolor{palepink}{rgb}{0.98, 0.85, 0.87}

\usepackage[colorlinks	= true,
raiselinks	= true,
linkcolor	= darkblue,
citecolor	= Mahogany,
urlcolor	= ForestGreen,
pdfauthor	= {Sasan Vakili},
pdftitle	= {},
pdfkeywords	= {},
pdfsubject	= {},
plainpages	= false]{hyperref}

\usepackage{amsfonts,dsfont,mathtools, mathrsfs,amsthm, amssymb,amsmath,mathdots} 
\usepackage{dsfont}
\usepackage{mathabx}
\usepackage{bbm}
\usepackage{accents}
\usepackage{adjustbox}
\usepackage{siunitx}

\usepackage{fancyhdr,mdframed}

\usepackage[T1]{fontenc}
\usepackage{algcompatible}

\usepackage{algorithm}
\usepackage{algorithmicx,multirow, titlecaps}

\usepackage{multirow}
\usepackage{bigstrut}
\usepackage{wrapfig}
\usepackage[font=small]{caption}
\usepackage{nicefrac}

\usepackage{epsfig}
\usepackage{graphicx}
\usepackage{float}
\usepackage{placeins} 

\usepackage{stackengine}
\usepackage{booktabs}

\usepackage{algpseudocode}

\usepackage{enumitem}
\setlist{noitemsep, topsep=0cm}
\usepackage{tikz}
\usepackage[utf8]{inputenc}
\usepackage{lscape}

\usepackage{colortbl}
\usepackage{svg}

\allowdisplaybreaks
\date{\today}
\makeatletter
\def\@settitle{\begin{center}%
		\baselineskip14\p@\relax
		\normalfont\LARGE\scshape\bfseries
		\@title
	\end{center}%
}

\def\@setauthors{%
  \begingroup
  \def\thanks{\protect\thanks@warning}%
  \trivlist
  \centering\small \@topsep30\p@\relax
  \advance\@topsep by -\baselineskip
  \item\relax
  \author@andify\authors
  \def\\{\protect\linebreak}%
  \authors%
  \ifx\@empty\contribs
  \else
    ,\penalty-3 \space \@setcontribs
    \@closetoccontribs
  \fi
  \endtrivlist
  \endgroup
}

\makeatother
\makeatletter

\def\subsection{\@startsection{subsection}{2}%
	\z@{.5\linespacing\@plus.7\linespacing}{.5\linespacing}%
	{\normalfont\large\bfseries}}

\def\subsubsection{\@startsection{subsubsection}{3}%
	\z@{.5\linespacing\@plus.7\linespacing}{.5\linespacing}%
	{\normalfont\itshape}}

\newtheorem{theorem}{Theorem}[section]
\newtheorem*{theorem*}{Theorem}
\newtheorem{definition}{Definition}
\newtheorem*{definition*}{Definition}

\newtheorem{proposition}[theorem]{Proposition}
\newtheorem{corollary}[theorem]{Corollary}
\newtheorem{lemma}[theorem]{Lemma}
\newtheorem{remark}[theorem]{Remark}
\newtheorem{example}{Example}
\newtheorem*{example*}{Example}

\newtheorem{problem}{Problem}
\newtheorem*{problem*}{Problem}
\DeclareMathAlphabet{\mathbbold}{U}{bbold}{m}{n}

\newcommand*{\sumOp}{\operatornamewithlimits{\sum}\limits}

\newcommand{\binarySum}[2]{\!\!\!\!\!\!\!\!\!\!\!{\sumOp_{\;\;\;\;\;\;\; {#1}} \!\!\!\!\!\!\!\!{#2}}}
\newcommand*{\prodOp}{\operatornamewithlimits{\prod}\limits}

\newcommand*{\minOp}{\operatornamewithlimits{min}\limits}

\newcommand*{\limOp}{\operatornamewithlimits{lim}\limits}

\newcommand{\drm}{\mathrm{d}}

\newcommand{\nth}{{\text{\tiny{th}}}}
\newcommand{\trace}{\mathrm{tr}} 
\newcommand{\tr}{\!^{\scalebox{0.65}{$\mathsf{T}$}}}
\newcommand{\inv}{\!\!^{\scalebox{0.65}{$-1$}}}
\newcommand{\gradient}{\nabla}

\newcommand{\Expectation}[1]{\mathbb{E}[#1]}
\newcommand{\bigExpectation}[1]{\mathbb{E}\big[#1\big]}
\newcommand{\autoExpectation}[1]{\mathbb{E}\left[#1\right]}

\newcommand{\diag}{\mathrm{diag}}

\newcommand{\intZ}{{\mathbb{Z}}}
\newcommand{\Real}{{\mathbb{R}}}
\newcommand{\Complex}{{\mathbb{C}}}
\newcommand{\Symm}{{\mathbb{S}}}
\newcommand{\NormalDist}{{\mathcal{N}}}
\newcommand{\UniformDist}{{\mathcal{U}}}
\newcommand{\eye}{\mathbb{I}}

\renewcommand*{\exp}[1]{\mathrm{exp}{\big(#1\big)}}
\newcommand{\expS}[1]{\mathrm{exp}{(#1)}}
\newcommand{\norm}[1]{\left\lVert#1\right\rVert}

\newcommand{\innerS}[2]{{\langle {#1,#2} \rangle}}
\newcommand{\characteristicFunc}[1]{\varphi{\big(#1\big)}}
\newcommand{\characteristicFuncS}[1]{\varphi{(#1)}}

\newcommand{\Pdist}{{\mathbb{P}}}
\newcommand{\zero}{0}
\newcommand{\zeromx}{0}

\newcommand{\one}{{\scalebox{1.05}{$\mathbf{1}$}}}

\newcommand{\Ubb}{{\mathbb{U}}}
\newcommand{\matrixSqrt}[1]{{#1^{^{\frac{1}{2}}}}}

\newcommand{\matrixGradient}[1]{\partial_{_{#1}}\!}

\newcommand{\setProject}[1]{{\mathcal{P}_{#1}}}

\newcommand{\fracComma}{\raisebox{0.5ex}{,}}

\newcommand{\trueFunc}{{h}}

\newcommand{\vecTheta}{{\mathrm{\theta}}}

\newcommand{\Ocal}{{\mathcal{O}}}
\newcommand{\Ntheta}{\mathrm{N}}
\newcommand{\vecY}{{\mathrm{y}}}

\newcommand{\vecV}{{\mathrm{v}}}
\newcommand{\matrixPhi}{\mathrm{\Phi}}
\newcommand{\matrixPhibar}{\bar{\mathrm{\Phi}}}

\newcommand{\matrixH}{\mathrm{H}}

\newcommand{\MuTheta}[1]{\mu_{\vecTheta_{#1}}}
\newcommand{\SigmaTheta}[1]{\Sigma_{\vecTheta_{#1}}}

\newcommand{\sigmaTheta}[1]{\sigma_{\vecTheta_{#1}}}

\newcommand{\Trajectory}{T}
\newcommand{\vecX}{{\mathrm{x}}}
\newcommand{\nX}{n_{\scalebox{0.65}{$\mathrm{x}$}}}
\newcommand{\nY}{n_{\scalebox{0.65}{$\mathrm{y}$}}}
\newcommand{\SigmaX}[1]{\Sigma_{\vecX_{#1}}}
\newcommand{\MuX}[1]{\mu_{\vecX_{#1}}}
\newcommand{\A}{{\mathrm{A}}}
\newcommand{\B}{{\mathrm{B}}}
\newcommand{\vecU}{{\mathrm{u}}}
\newcommand{\nU}{n_{\scalebox{0.65}{$\mathrm{u}$}}}
\newcommand{\vecW}{{\mathrm{w}}}
\newcommand{\SigmaW}[1]{\Sigma_{\vecW_{#1}}}
\newcommand{\Abar}{{\bar{\A}}}
\newcommand{\Bbar}{{\bar{\B}}}
\newcommand{\vecXbar}{{\bar{\vecX}}}
\newcommand{\vecUbar}{{\bar{\vecU}}}
\newcommand{\vecWbar}{{\bar{\vecW}}}
\newcommand{\SigmaWbar}[1]{\Sigma_{\vecWbar_{#1}}}
\newcommand{\vecYbar}{{\bar{\vecY}}}
\newcommand{\vecVbar}{{\bar{\vecV}}}

\newcommand{\SigmaVbar}[1]{\Sigma_{\vecVbar_{#1}}}

\newcommand{\iteration}{\mathrm{k}}

\newcommand{\sigmaV}[1]{\sigma_{\vecV_{#1}}}
\newcommand{\matrixM}{{\mathrm{M}}}
\newcommand{\matrixV}{{\mathrm{V}}}
\newcommand{\matrixS}{{\mathrm{S}}}
\newcommand{\DeltaPhi}{{\Delta\Phi}}
\newcommand{\deltaPhi}{{\Delta\phi}}
\newcommand{\DeltaTheta}{{\Delta\vecTheta}}
\newcommand{\nf}{n_{\scalebox{0.65}{$f$}}}
\newcommand{\Cost}{{\mathcal{J}}}

\newcommand{\vecC}{\mathrm{c}}

\newcommand{\matrixR}{{\mathrm{R}}}

\newcommand{\matrixPhiApp}{\Tilde{\matrixPhi}}
\newcommand{\vecPhiApp}{\Tilde{\phi}}

\newcommand{\vecR}{\mathrm{r}}
\newcommand{\vecS}{\mathrm{s}}
\newcommand{\SigmaPhi}{\Sigma_{\phi}}
\newcommand{\MuPhi}{\mu_{\phi}}
\newcommand{\MuPhiElements}[1]{\mu_{\phi_{#1}}}
\newcommand{\SigmaPhiElements}[1]{\Sigma_{\phi_{#1}}}

\newcommand{\batch}{\tau}

\title[Optimal Bayesian Affine Estimator and Active Learning for the Wiener Model]{Optimal Bayesian Affine Estimator and \\ Active Learning for the Wiener Model}

\author[S. Vakili]{Sasan Vakili}
\author[M. {Mazo Jr.}]{Manuel {Mazo Jr.}}
\author[P. {Mohajerin Esfahani}]{Peyman {Mohajerin Esfahani}}

\thanks{The authors are with the Delft Center for Systems and Control, Delft University of Technology, Delft, The Netherlands. Emails: \href{mailto:S.Vakili@tudelft.nl} {\texttt{S.Vakili@tudelft.nl}}; \href{mailto:M.Mazo@tudelft.nl} {\texttt{M.Mazo@tudelft.nl}}; \href{mailto:P.MohajerinEsfahani@tudelft.nl} {\texttt{P.MohajerinEsfahani@tudelft.nl}}.
This work was supported by the European Union's Horizon 2020 Research and Innovation Programme through the Marie Sk\l{}odowska-Curie Grant under Agreement 956200 and the ERC Starting Grant TRUST-949796.}

\begin{document}
\begin{abstract}
This paper presents a Bayesian estimation framework for Wiener models, focusing on learning nonlinear output functions under known linear state dynamics. We derive a closed-form optimal affine estimator for the unknown parameters, characterized by the so-called ``dynamic basis statistics''~(DBS). Several features of the proposed estimator are studied, including Bayesian unbiasedness, closed-form posterior statistics, error monotonicity in trajectory length, and consistency condition (also known as persistent excitation). In the special case of Fourier basis functions, we demonstrate that the closed-form description is computationally available, as the Fourier DBS enjoys explicit expressions. Furthermore, we identify an inherent inconsistency in the Fourier bases for single-trajectory measurements, regardless of the input excitation. Leveraging the closed-form estimation error, we develop an active learning algorithm synthesizing input signals to minimize estimation error. Numerical experiments validate the efficacy of our approach, showing significant improvements over traditional regularized least-squares methods.
\end{abstract}
\maketitle
\section{Introduction} \label{sec:Introduction}
System modelling and identification quality are essential for practical analysis, prediction, and control of complex phenomena across disciplines. Classical system identification techniques combine theoretical frameworks with empirical data to select appropriate model structures and estimate parameters for constructing accurate models of dynamic systems~\cite{soderstrom_system_1989, Ljung1999SysId, Deistler2001}. While linear system identification has a well-established foundation, identifying nonlinear systems remains an evolving area of research~\cite{Nelles2020}.

Parametric approaches for nonlinear system identification, such as extending linear state-space models to incorporate polynomial nonlinear terms~\cite{paduart2010identification}, have shown significant achievements~\cite{ljung2010perspectives}. However, these methods require a trade-off between bias and variance in model order selection~\cite{pillonetto2014kernel, schoukens2019nonlinear}. Recent regularization-based approaches address this challenge by exploring high-dimensional search spaces through kernel-based methods~\cite{chiuso2019system, pillonetto2022regularized}. These techniques enhance robustness in model selection by using continuous regularization parameters instead of discrete model orders. Furthermore, leveraging infinite-dimensional reproducing kernel Hilbert spaces via Gaussian process models~\cite{rasmussen_gaussian_2006} offers a probabilistic framework for nonlinear system identification~\cite{pillonetto2010new}, allowing prior knowledge about system properties, such as stability and smoothness, in the identification process.

While robust for system identification, kernel-based methods are computationally expensive for large-scale problems due to operations like matrix inversion in high dimensions. To address this challenge, randomized low-dimensional feature spaces, particularly Fourier basis functions, were introduced to accelerate kernel machine training~\cite{rahimi2007random}. Various techniques, including probabilistic variational methods~\cite{hensman2018variational}, have since been proposed to approximate radial basis kernels using random Fourier features~\cite{liu2021random}. Additionally, approximations have been developed to handle input noise in Gaussian process regression~\cite{girard2002gaussian, mchutchon2011gaussian}, enhancing applicability to real-world scenarios. This work addresses similar challenges for identifying an unknown nonlinear function influenced by time-varying correlated noise in its inputs. We propose an optimal Bayesian estimator for nonlinear functions represented as finite combinations of basis functions, focusing on Fourier bases.

Our problem can also be classified within the block-oriented nonlinear models, specifically Wiener systems characterized by a linear process followed by a static nonlinear observation model~\cite{schoukens2017identification}. Existing Wiener system identification techniques span a range of methodologies. Some approaches use Gaussian process models for the static nonlinear block and approximate posterior densities using Markov Chain Monte Carlo (MCMC) methods~\cite{lindsten2013bayesian, risuleo2020nonparametric}. Others model the static nonlinearity as a polynomial of known order or approximate it as a linear combination of predefined basis functions, employing the Prediction Error Method (PEM) or Gaussian sum filtering with Expectation-Maximization (EM) algorithms for inference~\cite{bottegai2017identification, cedeno2024identification, totterman2009support}. Additionally, studies on parameter estimation consistency using PEM and Maximum Likelihood under various noise assumptions~\cite{hagenblad1999aspects, hagenblad2008maximum} demonstrate that traditional least-squares methods can lead to biased estimates when process noise is present.

Unlike these existing techniques, our work focuses solely on the Bayesian estimation of static nonlinear observation parameters in Wiener systems with known linear time-varying dynamics affected by process and measurement noise. This problem arises in many applications, such as robot mapping in unknown environments. For instance, Autonomous Underwater Vehicles (AUVs) mapping the seabed in deep-sea environments face challenges due to the interplay between vehicle dynamics and unknown nonlinear seabed observation models. Our proposed Bayesian Minimum Mean Square Error (MMSE) affine estimator analytically computes estimates and estimation errors. By accounting for process noise correlations over time through information gained from the covariance of the observation model, our method avoids divergence issues observed in \emph{approximate prediction error method}.

While our Bayesian MMSE affine estimator addresses parameter estimation robustness against noise correlations, the quality of identification depends heavily on the choice of input signals used to excite the system. Active learning and optimal input design maximize information gain by selecting informative samples~\cite{settles2012active} or constructing input signals that optimize experimental criteria like Fisher or mutual information~\cite{fedorov1972theory, mehra1976synthesis}. Numerous studies have explored strategies for optimizing input selection for both linear~\cite{lindqvist2001identification, wagenmaker2020active} and nonlinear systems~\cite{gevers2012experiment, valenzuela2013optimal, valenzuela2017robust, mania2022active}, where different experimental design criteria have led to varied approaches. In this work, we derive an optimal input design by directly minimizing the analytical estimation error of our parameter estimates, specifically by minimizing the trace of the estimate covariance matrix. Integrating this optimal input design with our Bayesian MMSE affine estimator prior to conducting experiments enhances the efficacy of identifying static nonlinear observation parameters under correlated noise conditions.
\newpage 

\paragraph{\bf Contributions}
The main contributions are as follows:
\begin{itemize}
\item \textbf{Optimal Bayesian affine estimator:} 
Introducing a Bayesian setting, we derive the closed-form solution of the optimal affine estimator for the unknown parameters of the output function (Theorem~\ref{theorem:MMSE_LinearEstimator}), which is characterized in terms of the so-called ``\emph{dynamic basis statistics}'' (DBS).
\item \textbf{Optimal estimator features:} 
The proposed optimal estimator enjoys the following properties: (i) Bayesian unbiasedness (Proposition~\ref{proposition:Unbiasedness}), (ii) Closed-form updates for posterior statistics (Remark~\ref{remark:Posterior}), (iii) Monotonic error reduction (Corollary~\ref{corollary:Bayesian_Error_Reduction}), and (iv) Consistency under specific conditions (Proposition~\ref{proposition:Multi_Batch_Bayesian_MMSE_consistency}).
\item \textbf{Fourier basis explicit solution:} 
The generic closed-form solution of the Bayesian estimator requires the respective DBS of the basis functions, which can be computationally demanding. We show that these statistics admit explicit expressions in the special case of Fourier basis~(Lemma~\ref{lemma:Fourier_basis_Explicit_DBS}). We further identify an inherent inconsistency of single-trajectory measurements for the Fourier bases, irrespective of the input trajectory (even if unbounded and persistently excited), when the underlying dynamics are stochastically unstable~(Proposition~\ref{proposition:Bayesian_MMSE_consistency}).
\item \textbf{Active learning:} Leveraging the closed-form description of the estimation error, we propose a first-order algorithm to actively design an input signal that locally minimizes it.
\end{itemize}
The theoretical results are validated through extensive numerical experiments, demonstrating the superiority of our Bayesian estimator with active learning over classical regularized least-squares methods (cf.\,Figure~\ref{Fig:ErrorDiff_Histogram}). To facilitate reproducibility, we provide an open-source MATLAB library available at \href{https://github.com/sasanvakili/Bayesian4Wiener}{https://github.com/sasanvakili/Bayesian4Wiener}.
\paragraph{\bf Orgnaiziation}
Section~\ref{sec:Problem_Description} introduces the modelling setting and problem formulation. Section~\ref{sec:Solution_Approach} presents the solution approaches: classical regularized least-squares, Bayesian MMSE affine estimator and its properties. Section~\ref{sec:Fourier_Basis} provides explicit expressions for the special case of Fourier basis functions and discusses the consistency condition. Section~\ref{sec:Active_Learning} describes a first-order algorithm for actively learning input signals. Section~\ref{sec:Numerical_Experiments} outlines an experimental setup to examine the consistency condition and compares the proposed Bayesian affine estimator across four benchmarks. Detailed proofs of mathematical statements are provided in the ``Technical Proofs'' subsection of each corresponding section.
\paragraph{\bf Notation}
Throughout this paper, $\intZ_{+}$, $\Real$, $\Real_{+}$, $\Real^{n \times m}$, and $\Symm_{+}^{n}$ denote the set of positive integers, the real numbers, nonnegative real numbers, $n \times m$ real matrices, and the space of all symmetric positive semidefinite matrices in $\Real^{n \times n}$, respectively. The symbol $\eye$ refers to the identity matrix, vectors are represented with lowercase letters (e.g., $\phi$), while matrices are represented with uppercase letters (e.g., $\matrixPhi$). Subscripts denote elements of a vector or matrix (e.g., $\MuPhiElements{n}^{t}$ for a vector and $\SigmaPhiElements{mn}^{tt'}$ for a matrix), while superscripts represent the time index of vector or matrix elements (e.g., $\MuPhi^{t}$ for a vector and $\SigmaPhi^{tt'}$ for a matrix). The trace operator is denoted by $\trace$, the transpose of a matrix $\A$ is denoted by $\A^{\tr}\!$, and $\diag(\A_1,\ldots,\A_k)$ represents a block-diagonal matrix with diagonal entries $\A_1,\ldots,\A_k$. The inner product of two vectors~$\vecX$ and $\vecY$ is given by $\innerS{\vecX}{\vecY} = \vecX^{\tr}\vecY$, and the respective 2-norm is $\norm{\vecX} = \sqrt{\innerS{\vecX}{\vecX}}$. For a matrix $\A \in \Real^{n \times n}$, the largest (smallest) absolute value of eigenvalues is denoted by $\lambda_{\max}(\A)$ ($\lambda_{\min}(\A)$). The conic inequality $\A \preceq \B$ means that the matrix difference $\B - \A$ is positive semidefinite, i.e., $\B - \A \succeq 0$. The notation $\Pdist(\mu,\Sigma)$ refers to an arbitrary distribution with mean $\mu$ and covariance matrix $\Sigma$, while a multivariate normal (Gaussian) distribution is denoted by $\NormalDist(\mu,\Sigma)$ and a uniform distribution over $[a,b]$ is denoted by $\UniformDist(a,b)$. The symbol $\sim$ stands for \emph{“distributed according to”}.
\section{Problem description} \label{sec:Problem_Description}
Consider a \emph{known} discrete-time linear time-varying dynamical system where the states at time $t$ are observed through an \emph{unknown} observation model
\begin{equation}\label{eq:Dynamic_Observation_Models}
\begin{aligned}
\vecX_{t+1} & = \A_{t}\vecX_{t}+\B_{t}\vecU_t+\vecW_{t+1}, \\
\vecY_{t} & = \trueFunc(\vecX_{t})  + \vecV_{t}.
\end{aligned}
\end{equation}
Here, $t=\{0, \hdots, \Trajectory\}$ represents the time index starting from $0$ and ending at time $\Trajectory$, $\vecX_{t} \in \Real^{\nX}$ is the vector of state variables, $\A_{t} \in \Real^{\nX \times \nX}$ is the state transition matrix, $\vecU_{t} \in \Real^{\nU}$ is the vector of inputs, $\B_{t} \in \Real^{\nX \times \nU}$ is the input matrix, and $\vecW_{t+1} \in \Real^{\nX}$ is the process noise, which has a distribution given by $\Pdist(\zero, \SigmaW{t+1})$. Additionally, the initial state $\vecX_{0}$ is characterized by a mean vector $\MuX{0}$ and covariance matrix $\SigmaX{0}$. Observations are made through the scalar output measurements $\vecY_t \in \Real$, while $\vecV_t \in \Real$ represents the measurement noise with a distribution $\Pdist(0, \sigmaV{t}^{2})$. The output function $\trueFunc \colon \Real^{\nX} \to \Real$ is defined as a finite linear combination of \emph{known} basis functions $\phi_{n} \colon \Real^{\nX} \to \Complex$: 
\begin{equation} \label{eq:Function_Space}
\trueFunc(\vecX) = \!\sumOp_{n=0}^{\Ntheta} \vecTheta_{n} \phi_{n}(\vecX) = \innerS{\phi(\vecX)}{\vecTheta},
\end{equation} 
where $\vecTheta$~$=$~$[\vecTheta_{0}, \hdots, \vecTheta_{\Ntheta}]^{\tr}\!$ is the vector of \emph{unknown} parameters, $\phi(\vecX) = [\phi_{0}(\vecX), \hdots, \phi_{\Ntheta}(\vecX)]^{\tr}\!$ is the vector of basis functions evaluated at $\vecX$, and $\Ntheta+1$ is the number of basis functions. We note that the setting~\eqref{eq:Dynamic_Observation_Models}, i.e., linear dynamics followed by nonlinear output function, is referred to as the {\em Wiener} model.

\begin{remark} [\textbf{Modelling setting}]
\label{remark:Modeling_Setting}
Two important points are worth noting regarding the modelling setting of this study:
\begin{enumerate} [label=(\roman*)]
\item \textbf{Multivariate output measurements:} 
For measurements in higher dimensions ($\vecY_t \in \Real^{\nY}$), the output function~\eqref{eq:Function_Space} extends to a vector form $\trueFunc(\vecX) = [\trueFunc^{1}(\vecX), \hdots, \trueFunc^{\nY}(\vecX) ]^{\tr}\!$, where the goal is to learn each function $\trueFunc^{i}(\vecX)$ separately, parameterized as in ~\eqref{eq:Function_Space}. A common assumption in many applications is that the parameters of each $\trueFunc^{i}(\vecX)$ are independent of those of other components, effectively reducing the learning of a multivariate output function to multiple single-variate outputs. Therefore, for the simplicity of the exposition, we focus on single-output measurements ($\nY = 1$) for the remainder of this study.
\item \textbf{Bayesian prior interpretation:} 
A fundamental aspect of the Bayesian framework is incorporating prior information about unknown parameters. This information is formalized mathematically through a probability distribution, which in this study is characterized solely by its mean and covariance, i.e., the prior distribution belongs to $\Pdist(\MuTheta{}, \SigmaTheta{})$, where $\MuTheta{}$ and $\SigmaTheta{}$ are given modelling parameters.
\end{enumerate}
\end{remark}

Let the process noise $\vecW_t$, the measurement noise $\vecV_t$, the initial state $\vecX_0$, and the vector of parameters $\vecTheta$ be independent of one another at all times. Our objective is to identify the model parameters $\vecTheta$ from the measurement data $\vecY_{t}$ at all time steps, represented as $\vecYbar = [\vecY_{0}, \hdots, \vecY_{\Trajectory}]^{\tr}\!$. This task can also be interpreted as estimating the parameters of a linear or nonlinear function influenced by time-varying correlated noise in its inputs. The problem can be formally described as follows:

\begin{problem}[\textbf{Bayesian estimator for Wiener model}]\label{problem:Problem}
Given measurements $\vecYbar \in \Real^{\Trajectory+1}$ and the parameters prior information $\Pdist(\MuTheta{}, \SigmaTheta{})$, design the optimal estimator $\hat{\vecTheta} \colon \Real^{\Trajectory+1} \to \Theta$ that minimizes the expected loss, $\min_{\hat{\vecTheta}(\cdot)} \Expectation{\ell (\vecTheta, \hat{\vecTheta}(\vecYbar))}$, where $\ell$ is a predefined loss function quantifying the discrepancy between the true parameters $\vecTheta$ and their estimate $\hat{\vecTheta}(\vecYbar)$.
\end{problem}
\section{Solution Approaches} \label{sec:Solution_Approach}
The unknown function $\trueFunc (\vecX_{t})$ in the observation model of~\eqref{eq:Dynamic_Observation_Models} is assumed to belong exactly to the hypothesis class defined in~\eqref{eq:Function_Space}. Consequently, the output function can be reformulated and expressed in \emph{lifted matrix} form as
\begin{equation} \label{eq:Lifted_Observation}
\vecYbar = \matrixPhi^{\tr}\! \vecTheta + \vecVbar,
\end{equation}
where $\vecYbar = [\vecY_{0}, \hdots, \vecY_{\Trajectory}]^{\tr}\!$, $\matrixPhi = [\phi(\vecX_{0}), \hdots, \phi(\vecX_{\Trajectory})]$ is the basis aggregation matrix, $\vecTheta = [\vecTheta_{0}, \hdots, \vecTheta_{\Ntheta}]^{\tr}\!$, and $\vecVbar = [\vecV_{0}, \hdots, \vecV_{\Trajectory}]^{\tr}\!$ is the measurement noise vector. The measurement noise follows $\vecVbar \sim \Pdist(\zero, \SigmaVbar{})$, where $\SigmaVbar{} = \diag(\sigmaV{0}^{2}, \hdots, \sigmaV{\Trajectory}^{2})$ assuming $\vecV_{t}$ are independent from each other at all times. Let us recall that the prior information about $\vecTheta$ is characterized by a probability distribution defined by its mean and covariance, $\Pdist(\MuTheta{}, \SigmaTheta{})$ (cf.\,Remark~\ref{remark:Modeling_Setting}(ii)). The propagation of the state trajectory of the dynamical system through the output basis function is a key object in characterizing our proposed Bayesian estimator. This concept is introduced next.
\begin{definition} [\textbf{Dynamic basis statistics}]
\label{definition:Dynamic_Basis_Statistics}
Let $\vecX_{t}$ be the dynamics trajectory of the system~\eqref{eq:Dynamic_Observation_Models}, and $\phi(x)$ be the set of basis functions of the output function~\eqref{eq:Function_Space}. The dynamic basis statistics (DBS), denoted by $(\MuPhi^{t}, \SigmaPhi^{tt'})$, is the mean and covariance of~$\phi(\vecX_{t})$ at two time instants $(t,t')$, i.e.,
\begin{equation} \label{eq:Random_Phi}
\MuPhi^{t} = \Expectation{\phi(\vecX_{t})}, \qquad \SigmaPhi^{tt'} = \Expectation{\phi(\vecX_{t})\phi^{\tr}\!\!(\vecX_{t'})} - 
\Expectation{\phi(\vecX_{t})}\Expectation{\phi(\vecX_{t'})}^{\tr}\!.
\end{equation}
\end{definition}
Given the randomness of the elements of $\matrixPhi$ as defined in Definition~\ref{definition:Dynamic_Basis_Statistics}, the observation model can be rewritten as $\vecYbar = (\Expectation{\matrixPhi}+\DeltaPhi)^{\tr}\! \vecTheta + \vecVbar$, where $\DeltaPhi$ is a \emph{zero-mean} random matrix. If $\DeltaPhi$ were deterministic, however, a solution to Problem~\ref{problem:Problem} could be obtained via the classical least-squares methods of supervised learning, as discussed in the next subsection.
\subsection{Classical regularized least-squares} \label{subsec:Least_Square}
Regularized least squares (RLS), also called Ridge Regression, identifies unknown parameters by extending the ordinary least-squares method with a penalty on the parameters, known as the $\mathcal{L}_{2}$ regularization~\cite[Ch.~3]{bishop_pattern_2006}. This regularization reduces model complexity, helping to avoid overfitting and improving generalization to new, unseen data. Since $\matrixPhi$ in observation model~\eqref{eq:Lifted_Observation} is a random matrix, as noted in~\eqref{eq:Random_Phi}, one can \emph{approximate} its columns by $\matrixPhiApp = [\vecPhiApp(\vecX_{0}),\hdots,\vecPhiApp(\vecX_{\Trajectory})]$ in two ways:
\begin{enumerate} [label=(\roman*)]
\item Dead reckoning least squares (DLS): 
$\vecPhiApp(\vecX_{t}) = \big[\phi_{0}(\Expectation{\vecX_{t}}), \hdots, \phi_{\Ntheta}(\Expectation{\vecX_{t}})\big]^{\tr}\!$.
\item Mean least squares (MLS): 
$\vecPhiApp(\vecX_{t}) = \MuPhi^{t} = \big[\Expectation{\phi_{0}(\vecX_{t})}, \hdots, \Expectation{\phi_{\Ntheta}(\vecX_{t})} \big]^{\tr}\!$.
\end{enumerate}
\begin{subequations}\label{eq:Regularized_LS}
Using either of these approximations, the regularized least-squares method provides an estimator by solving the optimization problem
\begin{equation} \label{eq:Ridge_Regress}
\minOp_{\vecTheta} \, \norm{\vecYbar - \matrixPhiApp^{\tr}\!\vecTheta}^2 + \lambda \norm{\vecTheta}^2,
\end{equation}
where $\lambda$ is a hyperparameter. The closed-form optimal solution yields a linear estimator with respect to the measurements $\vecYbar$~\cite[Ch.~3]{bishop_pattern_2006} as
\begin{equation} \label{eq:Ridge_Regress_Solution}
\hat{\vecTheta}_{\mathrm{LS}}( \vecYbar ) = (\matrixPhiApp\matrixPhiApp^{\tr}\! + \lambda \eye )^{\inv}\! \matrixPhiApp \vecYbar.
\end{equation}
\end{subequations}
This estimator requires inverting a square matrix with dimensions equal to the number of basis functions. Consequently, the computational complexity of~\eqref{eq:Ridge_Regress_Solution} is $\Ocal_{\mathrm{LS}}(\Ntheta^{2}\Trajectory + \Ntheta^{3})$, depending on the number of measurements and unknown parameters. This complexity further simplifies to $\Ocal_{\mathrm{LS}}(\Trajectory)$ when $\Ntheta \ll \Trajectory$. These \emph{approximate} approaches are used to evaluate the performance of this method through numerical experiments in Section~\ref{sec:Numerical_Experiments}.
\subsection{Optimal Bayesian affine estimator} \label{subsec:Bayesian_Estimation}
Given the input-output trajectory data, the optimal Bayesian estimator $\hat{\vecTheta} \colon \Real^{\Trajectory+1} \to \Theta$, which addresses Problem~\ref{problem:Problem}, is obtained by solving
$$
\hat{\vecTheta} ( \vecYbar ) = \arg\!\min_{\vartheta \in \Theta} \Expectation{\ell ( \vecTheta, \vartheta )| \vecYbar} = \arg\!\min_{\vartheta \in \Theta} \int_{\Theta} \! \ell ( \vecTheta, \vartheta ) p(\vecYbar | \vecTheta) p(\drm \vecTheta),
$$
where the second equality follows from Bayes' rule. Here, $\ell(\vecTheta, \hat{\vecTheta}(\vecYbar))$ is a loss function that defines the performance criterion for estimating the unknown parameters using $\hat{\vecTheta}(\vecYbar)$. The term $p(\vecYbar|\vecTheta)$ represents the likelihood, i.e., the probability density function of the data given the unknown parameters, which is derived from the relationship between the data and the parameters. Meanwhile, $p(\drm \vecTheta)$ is the prior probability density function of the unknown parameters. Consequently, the Bayesian estimation approach requires both a performance criterion and a prior distribution as its starting point.

Among various performance criteria, the mean squared loss function, defined as $\ell(\vecTheta, \vartheta) = \norm{\vecTheta - \vartheta}^{2}$, leads to the Minimum Mean Square Error (MMSE) estimate. This estimate corresponds to the mean of the posterior distribution of the unknown parameters $\vecTheta$ given the measurements~\cite[Ch.~4]{levy2008principles},~i.e.,
\begin{equation} \label{eq:Cond_Exp}
\hat{\vecTheta} ( \vecYbar ) = \Expectation{\vecTheta | \vecYbar} = \int_{\Theta} \vecTheta p(\vecYbar | \vecTheta) p(\drm \vecTheta).
\end{equation}
However, computing the posterior solution~\eqref{eq:Cond_Exp} is often challenging due to either an incomplete specification of the likelihood distribution $p(\vecYbar|\vecTheta)$ or because the integral does not have a closed-form solution. For certain special classes of joint distributions, such as Gaussian distributions, the posterior distribution admits an analytical solution. A fundamental classical result in these cases is that the mean of the posterior belongs to the class of affine estimators~\cite[Ch.~2]{barfoot_state_2017}. The distribution $\Pdist$ of a random vector $\nu \in \Real^{\nf}$ is called \emph{elliptical}, denoted as $\Pdist = \mathcal{E}^{\nf}_{\rho}(\mu, \Sigma)$, if its characteristic function is given by $\characteristicFuncS{f} = \expS{j\innerS{f}{\mu}}\rho(f^{\tr}\Sigma f)$, where $\mu \in \Real^{\nf}$ is a location parameter, $\Sigma \in \Symm_{+}^{\nf}$ is the dispersion matrix, and $\rho: \Real_{+} \to \Real$ is the characteristic generator~\cite[p.~107]{johnson1987multivariate}.

\begin{remark} [\textbf{MMSE estimate of elliptical distributions}]
\label{remark:MMSE-EllipticalDistributions}
If the joint distribution of the unknown parameters $\vecTheta$ and the measurements $\vecYbar$ is elliptical, then the conditional distribution $p(\vecTheta | \vecYbar)$ is elliptical, and its first moment, forming the optimal solution~\eqref{eq:Cond_Exp}, is affine in the variable $\vecYbar$~\cite[Thm.~5]{cambanis1981theory}.
\end{remark}

Inspired by Remark~\ref{remark:MMSE-EllipticalDistributions} and to enhance computational efficiency, we restrict the family of estimators in Problem~\ref{problem:Problem} to affine functions for the remainder of this study. The following theorem presents the optimal closed-form solution for this class of estimators.

\begin{theorem} [\textbf{Optimal Bayesian MMSE affine estimator}] \label{theorem:MMSE_LinearEstimator}
The optimal Bayesian MMSE affine estimator of Problem~\ref{problem:Problem} is of the form $\hat{\vecTheta}_{\mathrm{B}}( \vecYbar ) = \Psi^{\star} \vecYbar + \psi^{\star}$, where 
\begin{subequations}\label{eq:opt_Bayes}
\begin{equation} \label{eq:MMSE_LinearEstimator}
\Psi^{\star} = \SigmaTheta{}\matrixPhibar \big( \matrixPhibar^{\tr}\!\SigmaTheta{}\matrixPhibar + \matrixM + \SigmaVbar{} \big)^{\inv}\!, \qquad \psi^{\star} = \MuTheta{} - \Psi^{\star} \matrixPhibar^{\tr}\! \MuTheta{},
\end{equation}
$\matrixPhibar = [\MuPhi^{0}, \hdots, \MuPhi^{\Trajectory}]$, the $(t+1,t'+1)^{\nth}$ element of matrix $\matrixM$ is $\matrixM_{tt'} = \trace \big(\SigmaPhi^{tt'}(\SigmaTheta{} + \MuTheta{}\MuTheta{}^{\tr}) \big)$, in which $(\MuPhi^{t}, \SigmaPhi^{tt'})$ is the respective DBS of $\phi(\vecX_{t})$ in the sense of Definition~\ref{definition:Dynamic_Basis_Statistics}. Consequently, the optimal MMSE estimation error is
\begin{equation} \label{eq:MMSE_OptimalCost}
\Cost^{\star}_{\mathrm{B}} = \trace \big(\SigmaTheta{} - \SigmaTheta{}\matrixPhibar \big( \matrixPhibar^{\tr}\!\SigmaTheta{}\matrixPhibar + \matrixM + \SigmaVbar{} \big)^{\inv}\!\matrixPhibar^{\tr}\!\SigmaTheta{} \big).
\end{equation}
\end{subequations}
\end{theorem}

We note that the technical proofs of the theoretical results are provided in Section~\ref{subsec:Technical_Proof_Sec3}. The computational complexity of the optimal Bayesian MMSE affine estimator is
$\Ocal_{\mathrm{B}}(\Ntheta^{2}\Trajectory^{2}+\Trajectory^{3})$, depending on the number of measurements and unknown parameters. When $\Ntheta \ll \Trajectory$, the computational complexity simplifies to $\Ocal_{\mathrm{B}}(\Trajectory^{3})$, which is significantly higher than that of the DLS and MLS linear estimators discussed in Section~\ref{subsec:Least_Square}. While computationally more expensive, the Bayesian MMSE estimator accounts for process noise correlations over time and provides unbiased estimates relative to the prior.

\begin{proposition} [\textbf{Bayesian unbiasedness}] \label{proposition:Unbiasedness}
The optimal Bayesian MMSE affine estimator~\eqref{eq:MMSE_LinearEstimator} is Bayesian unbiased, i.e., $\Expectation{\vecTheta-\hat{\vecTheta}_{\mathrm{B}}(\vecYbar)} = 0$, where $\vecTheta$ represents the true unknown parameters and the expectation is taken with respect to the joint distribution of $(\vecTheta, \vecYbar)$.
\end{proposition}

Given that the estimator is unbiased with respect to the prior, one can update the prior with the result obtained from this MMSE estimation.

\begin{remark} [\textbf{Bayesian update via posterior distribution}] 
\label{remark:Posterior}
Using the measurements $\vecYbar$ and the proposed Bayesian MMSE affine estimator, the unknown parameters follow a posterior distribution, i.e., $\vecTheta | \vecYbar \sim \Pdist(\MuTheta{}^{^{\mathrm{pos}}}\!, \SigmaTheta{}^{^{\mathrm{pos}}} )$, with known mean and covariance, given by
\begin{equation} \label{eq:Param_Estimate_Dist}
\begin{cases}
\MuTheta{}^{^{\mathrm{pos}}} \!= \MuTheta{} + \Psi^{\star} (\vecYbar- \matrixPhibar^{\tr}\! \MuTheta{}), \\
\SigmaTheta{}^{^{\mathrm{pos}}} \!= \SigmaTheta{} - \SigmaTheta{} \matrixPhibar \big(\matrixPhibar^{\tr}\!\SigmaTheta{}\matrixPhibar + \matrixM + \SigmaVbar{} \big)^{\inv}\! \matrixPhibar^{\tr}\! \SigmaTheta{}. 
\end{cases}
\end{equation}
\end{remark}

The proposed optimal affine estimator achieves the \emph{minimum variance} among all affine estimators and reduces the variance relative to the prior distribution, i.e., $\SigmaTheta{} \succeq \SigmaTheta{}^{^{\mathrm{pos}}}$, as evident from~\eqref{eq:Param_Estimate_Dist}.

\begin{corollary} [\textbf{Bayesian error monotonicity}] \label{corollary:Bayesian_Error_Reduction}
Let $\Cost^{\star}_{\mathrm{B}}(t)$ be the optimal MMSE estimation error defined in~\eqref{eq:opt_Bayes} at time $t$ corresponding the measurement vector $\vecYbar = [\vecY_{0}, \hdots, \vecY_{t}]^{\tr}\!$. Then, with one extra measurement at time $(t+1)$, the estimation error decreases monotonically, i.e., $\Cost^{\star}_{\mathrm{B}}(t+1) \leq \Cost^{\star}_{\mathrm{B}}(t)$.
\end{corollary}

In the following section, we present Fourier basis functions as a particular case of our generic solution and further illustrate the efficacy and performance of that through numerical analyses in Section~\ref{sec:Numerical_Experiments}.
\subsection{Technical Proofs} \label{subsec:Technical_Proof_Sec3}
This subsection contains detailed proofs of the theoretical results introduced earlier.
\begin{proof}[Proof of Theorem~\ref{theorem:MMSE_LinearEstimator}] 
\label{proof:MMSE_LinearEstimator}
Let us denote $\matrixPhi = \matrixPhibar + \DeltaPhi$, where $\matrixPhibar = \Expectation{\matrixPhi}$ and $\DeltaPhi$ is a \emph{zero-mean} random matrix. Specifically, $\matrixPhibar = [\MuPhi^{0}, \hdots, \MuPhi^{\Trajectory}]$, where $\MuPhi^{t} = \Expectation{\phi(\vecX_{t})}$, as defined by the DBS in~\eqref{eq:Random_Phi}. Similarly, $\vecTheta = \MuTheta{} + \DeltaTheta$ with $\DeltaTheta \sim \Pdist(\zero, \SigmaTheta{})$. Expanding the MMSE, $\minOp_{\Psi,\, \psi}\, \Expectation{\norm{\vecTheta - \Psi \vecYbar - \psi}^2}$, replacing $\vecYbar$ with its model from~\eqref{eq:Lifted_Observation}, and decomposing $\matrixPhi$ and $\vecTheta$ results in
\begin{equation} \label{eq:MMSE_Expanded}
\begin{aligned} 
\minOp_{\Psi,\, \psi} \,\,
& \mathbb{E}\Big[
\big(\DeltaTheta - \Psi (\matrixPhibar+\DeltaPhi )^{\tr}\!\!\DeltaTheta - \Psi\vecVbar - \Psi\DeltaPhi^{\tr}\!\MuTheta{} \big)^{\!\!\tr}\! 
\big(\DeltaTheta - \Psi (\matrixPhibar+\DeltaPhi )^{\tr}\!\!\DeltaTheta - \Psi\vecVbar - \Psi\DeltaPhi^{\tr}\!\MuTheta{} \big)
\Big] \\
& -2 \mathbb{E}\Big[
\big(\DeltaTheta - \Psi ( \matrixPhibar+\DeltaPhi )^{\tr}\!\!\DeltaTheta - \Psi\vecVbar - \Psi \DeltaPhi^{\tr}\!\MuTheta{} \big)^{\!\!\tr}\! 
\big(\psi - \MuTheta{} + \Psi \matrixPhibar^{\tr}\! \MuTheta{} \big) 
\Big] \\
& + \mathbb{E}\Big[
\big(\psi - \MuTheta{} + \Psi\matrixPhibar^{\tr}\!\MuTheta{} \big)^{\!\!\tr}\!
\big(\psi - \MuTheta{} + \Psi\matrixPhibar^{\tr}\!\MuTheta{} \big)
\Big].
\end{aligned}
\end{equation}
The second term in the above optimization is zero because all random variables are zero-mean, i.e., $\Expectation{\DeltaTheta} = \zero$, $\Expectation{\DeltaPhi} = \zeromx$, $\Expectation{\vecVbar} = \zero$, $\Expectation{\DeltaTheta^{\tr}\!\DeltaPhi} = 0$ from the stochastic independence of $\vecTheta$ and $\vecW_t$ at all times, and $\psi - \MuTheta{} + \Psi\matrixPhibar^{\tr}\!\MuTheta{}$ is a constant. In addition, the last term in~\eqref{eq:MMSE_Expanded} is zero if $\psi = \MuTheta{} - \Psi\matrixPhibar^{\tr}\!\MuTheta{}$.
Therefore, minimizing~\eqref{eq:MMSE_Expanded} over the variable $\Psi$ results in 
$$
\psi^{\star} = \MuTheta{} - \Psi^{\star}\matrixPhibar^{\tr}\!\MuTheta{}.
$$
Expanding the first term of~\eqref{eq:MMSE_Expanded}, applying the trace operator, noting that all the random variables are independent, and after some algebraic manipulation, the problem~\eqref{eq:MMSE_Expanded} reduces to minimizing the following cost function over $\Psi$:
$$
\Cost(\Psi) = \trace \big(\SigmaTheta{} - 2 \Psi^{\tr}\!\SigmaTheta{}\matrixPhibar + \Psi^{\tr}\!\Psi(\matrixPhibar^{\tr}\!\SigmaTheta{}\matrixPhibar + \matrixH + \matrixV + \SigmaVbar{})\big),
$$
where $\matrixH = \Expectation{\DeltaPhi^{\tr}\!\DeltaTheta\DeltaTheta^{\tr}\!\DeltaPhi}$ and $\matrixV = \Expectation{\DeltaPhi^{\tr}\!\MuTheta{}\MuTheta{}^{\tr}\DeltaPhi}$. The minimum of $\Cost(\Psi)$ is obtained from setting its partial derivative to zero, hence,
\begin{align*}
\dfrac{\partial \Cost}{\partial \Psi} \Big|_{\Psi=\Psi^{\star}} & \!\!= -2\SigmaTheta{}\matrixPhibar+2\Psi^{\star} \big(\matrixPhibar^{\tr}\!\SigmaTheta{}\matrixPhibar + \matrixH + \matrixV + \SigmaVbar{} \big) = 0 \\
& \Longleftrightarrow \Psi^{\star} = \SigmaTheta{}\matrixPhibar \big( \matrixPhibar^{\tr}\!\SigmaTheta{}\matrixPhibar + \matrixH + \matrixV + \SigmaVbar{} \big)^{\inv}\!\!.
\end{align*}
Since $\SigmaVbar{} = \diag(\sigmaV{0}^{2}, \hdots, \sigmaV{\Trajectory}^{2})$, we have $\matrixPhibar^{\tr}\!\SigmaTheta{}\matrixPhibar + \matrixH + \matrixV + \SigmaVbar{} \succ 0$, i.e., that the above matrix is invertible and $\Psi^{\star}$ has a unique solution. It remains to find the elements of matrices $\matrixH$ and $\matrixV$. The $(t+1,t'+1)^{\nth}$ elements of matrices $\matrixH$ and $\matrixV$ is calculated based on the $(t+1)^{\nth}$ and $(t'+1)^{\nth}$ columns of matrix $\DeltaPhi$ defined as $\deltaPhi(\vecX_{t})$ and $\deltaPhi(\vecX_{t'})$, respectively. Applying the trace operator to each of their elements,
$$
\begin{aligned}
\matrixH_{tt'} & = \Expectation{\deltaPhi^{\tr}\!\!(\vecX_{t})\DeltaTheta\DeltaTheta^{\tr}\!\deltaPhi(\vecX_{t'})} = \bigExpectation{\trace \big(\deltaPhi^{\tr}\!\!(\vecX_{t})\DeltaTheta\DeltaTheta^{\tr}\!\deltaPhi(\vecX_{t'}) \big)} = \bigExpectation{\trace \big(\deltaPhi(\vecX_{t'})\deltaPhi^{\tr}\!\!(\vecX_{t})\DeltaTheta\DeltaTheta^{\tr} \big)}, \\
\matrixV_{\!tt'} & = \Expectation{\deltaPhi^{\tr}\!\!(\vecX_{t})\MuTheta{}\MuTheta{}^{\tr}\deltaPhi(\vecX_{t'})} = \bigExpectation{\trace \big(\deltaPhi^{\tr}\!\!(\vecX_{t})\MuTheta{}\MuTheta{}^{\tr}\deltaPhi(\vecX_{t'}) \big)} = \bigExpectation{\trace \big(\deltaPhi(\vecX_{t'})\deltaPhi^{\tr}\!\!(\vecX_{t})\MuTheta{}\MuTheta{}^{\tr} \big)}.
\end{aligned}
$$
Noting that $\deltaPhi(\vecX_{t'})\deltaPhi^{\tr}\!\!(\vecX_{t})$ and $\DeltaTheta$ are independent, one could observe that
\begin{align*}
\matrixH_{tt'} 
& = \trace \big( \bigExpectation{\deltaPhi(\vecX_{t'})\deltaPhi^{\tr}\!\!(\vecX_{t})}\SigmaTheta{} \big), \\
\matrixV_{\!tt'} 
& = \trace \big( \bigExpectation{\deltaPhi(\vecX_{t'})\deltaPhi^{\tr}\!\!(\vecX_{t})}\MuTheta{}\MuTheta{}^{\tr} \big).
\end{align*}
Finally, rewriting $\deltaPhi(\vecX_{t}) = \phi(\vecX_{t}) - \Expectation{\phi(\vecX_{t})}$, it is straightforward to obtain
$$
\Expectation{\deltaPhi(\vecX_{t'})\deltaPhi^{\tr}\!\!(\vecX_{t})} 
= \Expectation{\phi(\vecX_{t'})\phi^{\tr}\!\!(\vecX_{t})} - \Expectation{\phi(\vecX_{t'})}\Expectation{\phi(\vecX_{t})}^{\tr}\!,
$$
hence, we have
$$
\begin{aligned}
\matrixH_{tt'} & = \trace \big( (\Expectation{\phi(\vecX_{t'})\phi^{\tr}\!\!(\vecX_{t})} - \Expectation{\phi(\vecX_{t'})}\Expectation{\phi(\vecX_{t})}^{\tr})\SigmaTheta{} \big), \\
\matrixV_{\!tt'} & = \trace \big( (\Expectation{\phi(\vecX_{t'})\phi^{\tr}\!\!(\vecX_{t})} - \Expectation{\phi(\vecX_{t'})}\Expectation{\phi(\vecX_{t})}^{\tr} )\MuTheta{}\MuTheta{}^{\tr} \big).
\end{aligned}
$$
Introducing matrix $\matrixM \coloneqq \matrixH + \matrixV$, its $(t+1,t'+1)^{\nth}$ element is 
$$
\matrixM_{tt'} = \matrixH_{tt'}+\matrixV_{\!tt'} = \trace \big(\SigmaPhi^{t't}(\SigmaTheta{} + \MuTheta{}\MuTheta{}^{\tr}) \big),
$$
where $\SigmaPhi^{t't} = \Expectation{\phi(\vecX_{t'})\phi^{\tr}\!\!(\vecX_{t})} - \Expectation{\phi(\vecX_{t'})}\Expectation{\phi(\vecX_{t})}^{\tr}\!$. Since 
$$
\trace \big(\SigmaPhi^{t't}(\SigmaTheta{} + \MuTheta{}\MuTheta{}^{\tr}) \big) = \trace \big(\SigmaPhi^{tt'}(\SigmaTheta{} + \MuTheta{}\MuTheta{}^{\tr})\big),
$$
we write $\matrixM_{tt'} = \trace \big(\SigmaPhi^{tt'}(\SigmaTheta{} + \MuTheta{}\MuTheta{}^{\tr}) \big)$. Ultimately, substituting $\Psi^{\star} = \SigmaTheta{}\matrixPhibar \big( \matrixPhibar^{\tr}\!\SigmaTheta{}\matrixPhibar + \matrixM + \SigmaVbar{}\big)^{\inv}\!$ and $\psi^{\star} = \MuTheta{} - \Psi^{\star}\matrixPhibar^{\tr}\!\MuTheta{}$ in~\eqref{eq:MMSE_Expanded} results in the last two terms to be zero and the optimal MMSE error $\Cost^{\star}_{\mathrm{B}}$ after a simple algebraic manipulation arrives at~\eqref{eq:MMSE_OptimalCost}, which concludes the proof.
\end{proof}
\begin{proof}[Proof of Proposition~\ref{proposition:Unbiasedness}]
\label{proof:Unbiasedness}
Showing that $\Expectation{\vecTheta - \hat{\vecTheta}_{\mathrm{B}}(\vecYbar)} = \zero$ indicates that the estimator is Bayesian unbiased. We substitute $\hat{\vecTheta}_{\mathrm{B}}(\vecYbar)$ with its optimal estimator~\eqref{eq:MMSE_LinearEstimator}, which leads to
$$
\begin{aligned}
\Expectation{\vecTheta-\hat{\vecTheta}_{\mathrm{B}}(\vecYbar)} = \Expectation{\vecTheta} - \Expectation{\Psi^{\star} \vecYbar + \psi^{\star}} 
= \MuTheta{} - \Psi^{\star}\Expectation{\vecYbar} - \MuTheta{} + \Psi^{\star} \matrixPhibar^{\tr}\! \MuTheta{}
= \Psi^{\star} \big(\matrixPhibar^{\tr}\! \MuTheta{} - \Expectation{\vecYbar}\big).
\end{aligned}
$$
Noting that $\Expectation{\vecYbar} = \Expectation{\matrixPhi}\MuTheta{}$ from the stochastic independence of $\vecTheta$ and $\vecW_t$ at all times, and that $\Expectation{\matrixPhi} = \matrixPhibar$, it follows that $\Expectation{\vecTheta - \hat{\vecTheta}_{\mathrm{B}}(\vecYbar)} = \zero$.
\end{proof}
\begin{proof}[Proof of Corollary~\ref{corollary:Bayesian_Error_Reduction}] 
\label{proof:Bayesian_Error_Reduction}
Let us define $\matrixR \coloneqq \matrixM+\SigmaVbar{}$, then the optimal MMSE estimation error in~\eqref{eq:MMSE_OptimalCost} is 
$$
\Cost^{\star}_{\mathrm{B}} = \trace \big( \SigmaTheta{} - \SigmaTheta{}\matrixPhibar \big( \matrixPhibar^{\tr}\!\SigmaTheta{}\matrixPhibar + \matrixR \big)^{\inv}\!\matrixPhibar^{\tr}\! \SigmaTheta{} \big).
$$
Applying the matrix inversion lemma~\cite{bernstein2009matrix}, we derive the equivalent expression
\begin{equation}
\label{eq:MMSE_OptimalCost_InverseLemma}
\SigmaTheta{} - \SigmaTheta{}\matrixPhibar \big( \matrixPhibar^{\tr}\!\SigmaTheta{}\matrixPhibar + \matrixR \big)^{\inv}\!\matrixPhibar^{\tr}\! \SigmaTheta{} = \big(\SigmaTheta{}^{\inv}\!+\matrixPhibar\matrixR^{\inv}\!\matrixPhibar^{\tr}\big)^{\inv}\!,
\end{equation}
which allows us to rewrite~\eqref{eq:MMSE_OptimalCost} as $\Cost^{\star}_{\mathrm{B}} = \trace \big( \big(\SigmaTheta{}^{\inv}\!+\matrixPhibar\matrixR^{\inv}\!\matrixPhibar^{\tr}\big)^{\inv} \big)$.
Let $\Cost^{\star}_{\mathrm{B}}(t+1)$ denote the estimation error at time $(t+1)$ using measurements $\vecYbar = [\vecY_{0}, \hdots, \vecY_{(t+1)}]^{\tr}\!$. Define matrices
$$
\matrixPhibar(t+1) \!=\! [\matrixPhibar(t), \, \MuPhi^{(t+1)}], \qquad
\matrixR(t+1) \!=\! 
\begin{bmatrix}
\matrixR(t) & \vecC(t+1) \\
\vecC^{\tr}\!(t+1) & \vecR(t+1)
\end{bmatrix}\!,
$$
with $\matrixPhibar(t) = [\MuPhi^{0}, \hdots,\MuPhi^{t}]$, $\vecC(t+1) = [\matrixM_{0(t+1)}, \hdots, \matrixM_{t(t+1)}]^{\tr}\!$, and $\vecR(t+1) = \matrixM_{(t+1)(t+1)}+\sigmaV{(t+1)}^{2}$.
Using the \emph{Schur complement}~\cite{bernstein2009matrix},
$$
\matrixR^{\inv}\!(t+1) = 
\begin{bmatrix}
\eye & -\matrixR^{\inv}\!(t) \vecC(t+1) \\
\zeromx & \eye
\end{bmatrix}
\begin{bmatrix}
\matrixR^{\inv}\!(t) & \zeromx \\
\zeromx & \big(\vecR(t+1) - \vecC^{\tr}\!(t+1) \matrixR^{\inv}\!(t)\vecC(t+1) \big)^{\inv}
\end{bmatrix}
\begin{bmatrix}
\eye & \zeromx \\
-\vecC^{\tr}\!(t+1) \matrixR^{\inv}\!(t) & \eye
\end{bmatrix}.
$$
Given $\matrixR(t+1) = \matrixM(t+1) + \SigmaVbar{}(t+1) \succ \zeromx$, one can observe that $\big(\vecR(t+1) - \vecC^{\tr}\!(t+1) \matrixR^{\inv}\!(t)\vecC(t+1) \big)^{\inv}\! > 0$.
We define $\matrixS(t+1) \coloneqq \matrixPhibar(t+1)\matrixR^{\inv}\!(t+1)\matrixPhibar^{\tr}\!(t+1)$, which decomposes as $\Delta\matrixS(t+1) \coloneqq \matrixS(t+1) - \matrixS(t)$, where $\matrixS(t) \coloneqq \matrixPhibar(t) \matrixR^{\inv}\!(t) \matrixPhibar^{\tr}\!(t)$, $\Delta\matrixS(t) = \gamma(t+1)^{\inv} \vecS(t+1)\vecS^{\tr}\!(t+1) \succeq \zeromx$,
$$
\vecS(t+1) = \matrixPhibar(t) \matrixR^{\inv}\!(t)\vecC(t+1) - \MuPhi^{(t+1)}, \qquad \gamma(t+1) = \vecR(t+1)-\vecC^{\tr}\!(t+1) \matrixR^{\inv}\!(t)\vecC(t+1).
$$
Since $\SigmaTheta{}^{\inv}\! \succeq \zeromx$, $\matrixS(t+1) \succeq \zeromx$, and $\matrixS(t) \succeq \zeromx$, we obtain 
$$
\big( \SigmaTheta{}^{\inv}\!+\matrixS(t)+\Delta\matrixS(t+1) \big)^{\inv}\! \preceq \big( \SigmaTheta{}^{\inv}\!+\matrixS(t) \big)^{\inv}\!.
$$
Thus, the estimation error decreases monotonically, i.e.,
$$
\Cost^{\star}_{\mathrm{B}}(t+1) = \trace \big( \big( \SigmaTheta{}^{\inv}\!+\matrixS(t)+\Delta\matrixS(t+1) \big)^{\inv} \big) 
\leq \trace \big( \big( \SigmaTheta{}^{\inv}\!+\matrixS(t) \big)^{\inv} \big) = \Cost^{\star}_{\mathrm{B}}(t).
$$
\end{proof}
\section{Fourier Basis} 
\label{sec:Fourier_Basis}
The question that arises concerns the extent to which the matrices $\matrixPhibar$ and $\matrixM$ of Theorem~\ref{theorem:MMSE_LinearEstimator} in~\eqref{eq:opt_Bayes} depend on the respective DBS $(\MuPhi^{t}, \SigmaPhi^{tt'})$ introduced in Definition~\ref{definition:Dynamic_Basis_Statistics}. By employing Fourier basis functions, one can leverage their unique properties to efficiently compute these expectations via their characteristic function. Furthermore, any sufficiently well-behaved function can be approximated as a sum of Fourier series~\cite{rudin_fourier_1962}, with the Discrete Fourier Transform (DFT) providing an efficient method for calculating the coefficients of these series. The Fourier basis function is defined as
\begin{equation} \label{eq:Fourier_Basis}
\begin{cases}
\phi_{0}(\vecX) = 1 & n = 0 \\
\phi_{n}(\vecX) = \binarySum{\ell\in\{-1,1\}}{\expS{j\innerS{\ell f_{n}}{\vecX}}} & n \geq 1,
\end{cases}
\end{equation}
where $f_{n} \in \Real^{\nX}$ represents a \emph{known} frequency, and $n$ denotes the frequency index. To ensure that the codomain of $\trueFunc$ in~\eqref{eq:Function_Space} remains in $\Real$, symmetry is imposed on the frequencies and their corresponding parameters to eliminate imaginary components. Specifically, the basis function is constructed such that identical parameters are assigned to terms with positive and negative frequencies, $f_{n}$ and $-f_{n}$. Additionally, $\phi_{0}(\vecX) = 1$ corresponds to a Fourier basis with $f_{0} = \zero$. Using the Fourier basis functions defined in~\eqref{eq:Fourier_Basis}, we derive explicit expressions for the elements of the mean vectors $\MuPhi^{t}$ and covariance matrices $\SigmaPhi^{tt'}$, which are required for the analytical formulation of the optimal Bayesian MMSE estimator outlined in Theorem~\ref{theorem:MMSE_LinearEstimator}. These expressions are presented in a compact form, which relies on the \emph{lifted matrix} representation introduced in the following section.
\subsection{Lifted process model}
\label{subsec:Lifted_Process_Model}
We represent the process model in~\eqref{eq:Dynamic_Observation_Models} for the entire trajectory in the \emph{lifted matrix} form
\begin{equation} \label{eq:Lifted_Dynamic}
\vecXbar = \Abar (\Bbar \vecUbar + \vecWbar),
\end{equation}
where $\vecXbar = [\vecX_0^{\tr}, \hdots, \vecX_{\Trajectory}^{\tr}]^{\tr}\!$ consists of the system states vector, $\vecUbar = [\MuX{0}^{\tr}, \vecU_0^{\tr}, \hdots, \vecU_{\Trajectory-1}^{\tr}]^{\tr}\!$ is the input vector with the initial state mean as its first element, and $\vecWbar = [\vecW_{0}^{\tr}, \vecW_1^{\tr}, \hdots, \vecW_{\Trajectory}^{\tr}]^{\tr}\!$ denotes the noise vector in which the first element corresponds to the uncertainty of the initial state. As such, $\vecW_{0} \sim \Pdist(\zero,\SigmaX{0})$ and $\vecWbar$ is a \emph{zero-mean} uncertainty, i.e., $\vecWbar \sim \Pdist(\zero,\SigmaWbar{})$, $\SigmaWbar{} \succeq 0$. The covariance matrix is diagonal only if $\vecW_{t}$ are independent, i.e., $\SigmaWbar{} = \diag(\SigmaX{0}, \SigmaW{1}, \hdots, \SigmaW{\Trajectory})$. Furthermore, the matrices $\Abar$ and $\Bbar$ have the following lower triangular and block-diagonal structures, respectively:
\begin{equation} \label{eq:matrixAandB}
\Abar = \begin{bmatrix} \eye & \zeromx & \zeromx & \hdots & \zeromx \\ 
\A_0 & \eye & \zeromx & \hdots & \vdots \\ 
\A_1 \A_0 & \A_1 & \eye & \ddots & \vdots \\ 
\vdots & \vdots &  \vdots & \ddots & \zeromx \\ 
\prodOp_{i=0}^{\Trajectory-1} \A_{i} & \prodOp_{i=1}^{\Trajectory-1} \A_{i} &  \hdots & \A_{\Trajectory-1} & \eye
\end{bmatrix}, \qquad 
\Bbar = \, \diag(\eye, \B_{0}, \hdots, \B_{\Trajectory-1}).
\end{equation}
The lifted matrix representation in~\eqref{eq:Lifted_Dynamic} provides a compact expression for the mean and covariance of the random variable $\vecX_{t}$ as a function of $\vecUbar$ and $\SigmaWbar{}$, given by $\vecX_{t} = \Abar_{t}(\Bbar \vecUbar + \vecWbar)$ and $\vecX_{t} \sim \Pdist(\Abar_{t}\Bbar\vecUbar, \Abar_{t}\SigmaWbar{}\Abar_{t}^{\tr})$, where $\Abar_{t}$ corresponds to the $(t+1)^\nth$ block-row of $\bar{\A}$ in~\eqref{eq:matrixAandB}, i.e.,
\begin{equation} \label{eq:matrixA_t}
\Abar_{t} = \begin{bmatrix}
\prodOp_{i=0}^{t-1} \A_{i} & \prodOp_{i=1}^{t-1} \A_{i} &  \hdots & \eye & \zeromx & \hdots & \zeromx
\end{bmatrix}.
\end{equation}
This representation serves as the foundation for deriving the DBS of Fourier bases presented next.
\subsection{Fourier basis statistics}
\label{subsec:Fourier_Basis_Expectation}
To compute the expectation of a Fourier basis, the concept of the characteristic function is directly relevant, as it provides a mechanism for deriving such expectations.

\begin{definition} [\textbf{Characteristic function}]
\label{definition:Characteristic_Function}
The characteristic function of a random vector is defined as the sign-reversed Fourier transform of its probability density function~\cite{gnedenko2018theory}, i.e., for a random vector $\nu \sim \Pdist(\zero, \eye)$, the characteristic function at frequency $f$ is given by $\characteristicFuncS{f} \coloneqq \Expectation{\expS{j\innerS{f}{\nu}}}$.
\end{definition}

Using characteristic functions, the following lemma derives explicit expressions for the elements of $\MuPhi^{t}$ and $\SigmaPhi^{tt'}$, which are essential for the optimal Bayesian MMSE affine estimator presented in Theorem~\ref{theorem:MMSE_LinearEstimator}.

\begin{lemma} [\textbf{Fourier DBS explicit expression}]
\label{lemma:Fourier_basis_Explicit_DBS}
For Fourier basis functions~\eqref{eq:Fourier_Basis} and the dynamics of~\eqref{eq:Dynamic_Observation_Models}, let the respective DBS introduced in Definition~\ref{definition:Dynamic_Basis_Statistics}, denoted by $(\MuPhi^{t}, \SigmaPhi^{tt'})$. Additionally, let $\Abar_{t}$ and $\Abar_{t'}$ represent the matrices defined analogously to~\eqref{eq:matrixA_t}. Then, the following holds:
\begin{enumerate} [label=(\roman*)]
\item {\bf DBS mean:} The first element of the mean vector is $\MuPhiElements{0}^{t} \!= 1$, and its $(n+1)^{\nth}$ element is
\begin{subequations}\label{eq:Fourier_DBS}
\begin{equation}
\label{eq:DBS_Mean}
\MuPhiElements{n}^{t} \!= \binarySum{\ell \in \{-1,1\}}{\expS{j\innerS{\Abar_{t}^{\tr} \ell f_{n}}{\Bbar\vecUbar}}\characteristicFuncS{\matrixSqrt{\SigmaWbar{}} \Abar_{t}^{\tr} \ell f_{n}}}, \quad n \geq 1.
\end{equation}

\item {\bf DBS covariance:} The $(m+1,n+1)^{\nth}$ element of the covariance matrix is $\SigmaPhiElements{mn}^{tt'} = 0$ if $m = 0$ or $n = 0$. For $m, n \geq 1$, this element is 

\begin{align} \label{eq:DBS_Covariance}
& \SigmaPhiElements{mn}^{tt'} \!= \binarySum{\ell,\ell'\in\{-1,1\}}{ \begin{multlined}[t] \expS{j\innerS{\Abar_{t}^{\tr}\ell f_{m}+\Abar_{t'}^{\tr}\ell'\! f_{n}}{\Bbar\vecUbar}} \\
\qquad \times \big(\characteristicFunc{\matrixSqrt{\SigmaWbar{}}\! (\Abar_{t}^{\tr}\ell f_{m}+\Abar_{t'}^{\tr}\ell'\! f_{n})} 
- \characteristicFuncS{\matrixSqrt{\SigmaWbar{}} \Abar_{t}^{\tr} \ell f_{m}} \characteristicFuncS{\matrixSqrt{\SigmaWbar{}} \Abar_{t'}^{\tr} \ell'\! f_{n}}\big).
\end{multlined}}
\end{align}
\end{subequations}
\end{enumerate}
\end{lemma}

The technical proofs of the theoretical statements presented in this section are deferred to Section~\ref{subsec:Technical_Proof_Sec4}. Using the explicit expressions provided in Lemma~\ref{lemma:Fourier_basis_Explicit_DBS}, one can analytically compute $\MuPhi^{t}$ and $\SigmaPhi^{tt'}$ when the process noise is drawn from known distributions.

\begin{example} [\textbf{Gaussian explicit characteristic}]
\label{example:Gaussian_Characteristic}
If the process noise is Gaussian, $\vecWbar \sim \NormalDist(\zero, \SigmaWbar{})$, the system states $\vecX_{t}$ follow a Gaussian distribution as $\vecX_{t} \sim \NormalDist(\Abar_{t}\Bbar\vecUbar, \Abar_{t}\SigmaWbar{}\Abar_{t}^{\tr})$, where $\Abar_{t}$ is defined in~\eqref{eq:matrixA_t}. In this case, the explicit expressions for~\eqref{eq:Fourier_DBS}, when $m, n \geq 1$, are obtained based on the analytical characteristic function of a multivariate normal distribution~\cite[Ch.~10]{florescu2013handbook},
\begin{equation}
\label{eq:Gaussian_Characteristic}
\begin{aligned}
\characteristicFuncS{\matrixSqrt{\SigmaWbar{}} \Abar_{t}^{\tr} \ell f_{n}} & = \exp{\!-\frac{1}{2} \ell f_{n}^{\tr}\Abar_{t} \SigmaWbar{} \Abar_{t}^{\tr} \ell f_{n}}, \\
\characteristicFunc{\matrixSqrt{\SigmaWbar{}}\! (\Abar_{t}^{\tr}\ell f_{m}+\Abar_{t'}^{\tr}\ell'\! f_{n})} & = \exp{\!-\frac{1}{2}(\ell f_{m}^{\tr}\Abar_{t}+\ell'\! f_{n}^{\tr}\Abar_{t'}) \SigmaWbar{} (\Abar_{t}^{\tr}\ell f_{m}+\Abar_{t'}^{\tr}\ell'\! f_{n})}.
\end{aligned}
\end{equation}
\end{example}

From the analytical expressions of the optimal estimator in Theorem~\ref{theorem:MMSE_LinearEstimator}, under the assumption of Gaussian process noise, we examine additional properties of the affine estimator, such as consistency, in the subsequent section.
\subsection{Bayesian MMSE affine estimator consistency}
\label{subsec:Consistency}
As noted in previous sections, the proposed Bayesian affine estimator~\eqref{eq:opt_Bayes} requires the computation of the inverse of the observation covariance matrix, which depends on the matrices $\matrixPhibar$ and $\matrixM$. These matrices, and consequently the optimal MMSE error~\eqref{eq:MMSE_OptimalCost}, are significantly influenced by the underlying dynamical system. The following proposition identifies the conditions under which the estimation error fails to converge to zero as the number of dependent samples from a trajectory length $\Trajectory$ approaches infinity, highlighting cases where the estimator is inconsistent.

\begin{proposition} [\textbf{Inherent inconsistency}] \label{proposition:Bayesian_MMSE_consistency}
Consider an unstable linear time-invariant system~\eqref{eq:Dynamic_Observation_Models}, where $\A_{t} = \A$ and $\lambda_{\max}(\A) \geq 1$. Let the process noise be Gaussian, i.e., $\vecWbar \sim \NormalDist(\zero, \SigmaWbar{})$, and have an observation model~\eqref{eq:Function_Space} with $\Ntheta \geq 1$ spanned by the Fourier basis~\eqref{eq:Fourier_Basis}. Then, for any input trajectory $\{\vecU_{t}\}_{t \in \intZ_{+}}$ (possibly unbounded and persistently excited), the optimal estimation error~\eqref{eq:MMSE_OptimalCost} is uniformaly away from zero, i.e., $\lim_{t \to \infty} \Cost^{\star}_{\mathrm{B}}(t) > 0$.
\end{proposition}

In cases where the underlying dynamical system is marginally stable or unstable, multiple \emph{statistically independent} trajectories of data help reduce the estimation error and ensure that the estimator remains consistent, provided trajectories are persistently excited and their number approaches infinity. We extend our formulation to accommodate $\batch$ multiple \emph{statistically independent} trajectories of data. For each trajectory $i$, the time index $t$ independently starts from $0$ and continues for a horizon $\Trajectory_{i}$, resulting in independent trajectories $\{0, \ldots, \Trajectory_{1}\}, \ldots, \{0, \ldots, \Trajectory_{\batch}\}$. In this extended formulation, the \emph{lifted matrix} forms of $\Abar$ and $\Bbar$ in the process model~\eqref{eq:Lifted_Dynamic} change to block-diagonal matrices, i.e., $\Abar = \diag(\Abar_{1}, \hdots, \Abar_{\batch})$ and $\Bbar = \diag(\Bbar_{1}, \hdots, \Bbar_{\batch})$, where $\Abar_{i}$ and $\Bbar_{i}$ are defined as in~\eqref{eq:matrixAandB} for each trajectory $i$. All other aspects of the formulation remain identical to the single-trajectory case, with their sizes extended according to the total number of data points $\Trajectory^{\batch} \!= (\Trajectory_{1}+1) + \cdots + (\Trajectory_{\batch}+1)$. The following proposition formally specifies the condition to provide consistency for the proposed optimal Bayesian MMSE affine estimator.

\begin{proposition} [\hspace{-1pt}\textbf{Consistency via independent trajectories}] \label{proposition:Multi_Batch_Bayesian_MMSE_consistency}
Let $\Cost^{\star}_{\mathrm{B}}(\batch, \Trajectory^{\batch})$ denote the optimal estimation error~\eqref{eq:MMSE_OptimalCost} using $\batch$ statistically independent trajectories (or ``batches'') with the the total length of $\Trajectory^{\batch}\!= (\Trajectory_{1}+1) + \cdots + (\Trajectory_{\batch}+1)$ data points, where each trajectory $i$ contains $(\Trajectory_{i}+1)$ measurements, i.e., $\vecYbar_{i} = [\vecY_{0}, \hdots, \vecY_{\Trajectory_{i}}]^{\tr}\!$. Then, for any prior distribution and any set of basis functions $\phi(\vecX)$ in~\eqref{eq:Function_Space}, the optimal Bayesian MMSE estimation error~\eqref{eq:MMSE_OptimalCost} converges to zero as $\batch$ tends to $\infty$ if the minimum eigenvalue of the so-called ``{information matrix}'' of the basis functions diverges to infinity, i.e.,
\begin{equation}\label{eq:Consistency_Condition}
\limOp_{\batch \to \infty} \lambda_{\min}\big(
\sumOp_{i=1}^{\batch} \frac{\MuPhi^{0}(i)\MuPhi^{0^{\tr}}\!(i)}{\matrixM_{00}(i)+\sigmaV{0}^{2}(i)}
\big) = \infty \implies
\limOp_{\batch \to \infty} \Cost^{\star}_{\mathrm{B}}(\batch, \Trajectory^{\batch}) = 0,
\end{equation}
where $\matrixM_{00}(i) = \trace \big(\SigmaPhi^{00}(i)(\SigmaTheta{} + \MuTheta{}\MuTheta{}^{\tr}) \big)$, $\sigmaV{0}^{2}(i)$ is the measurement noise variance, and $\MuPhi^{0}(i)$ and $\SigmaPhi^{00}(i)$ are the respective DBS as described in Definition~\ref{definition:Dynamic_Basis_Statistics},
$$
\MuPhi^{0}(i) = \Expectation{\phi(\vecX_{0}(i))}, \qquad 
\SigmaPhi^{00}(i) = \Expectation{\phi(\vecX_{0}(i))\phi^{\tr}\!\!(\vecX_{0}(i))} - \Expectation{\phi(\vecX_{0}(i))}\Expectation{\phi(\vecX_{0}(i))}^{\tr}\!,
$$
all evaluated at the initial state of the $i^{\nth}$ trajectory.
\end{proposition}

Condition~\eqref{eq:Consistency_Condition} shows that estimates of $\vecTheta$ from the optimal Bayesian MMSE affine estimator~\eqref{eq:opt_Bayes} converge to their true values if the initial states in \emph{statistically independent} trajectories are indeed persistently excited. However, the rate at which the estimation error decays also depends on the persistent excitation of dependent data within individual trajectories. Persistent excitation can be achieved through the strategic selection of inputs $\vecUbar$ by actively minimizing the optimal Bayesian MMSE error~\eqref{eq:MMSE_OptimalCost}. In the subsequent section, we address this input optimization framework, commonly referred to as active learning.
\subsection{Technical Proofs} \label{subsec:Technical_Proof_Sec4}
We provide detailed proofs supporting the theoretical statements of this section.
\begin{proof}[Proof of Lemma~\ref{lemma:Fourier_basis_Explicit_DBS}] 
\label{proof:Fourier_basis_Bayesian_MMSE_proof}
The random vector $\vecX_{t}$ follows the distribution $\vecX_{t} \sim \Pdist(\Abar_{t}\Bbar\vecUbar, \Abar_{t}\SigmaWbar{}\Abar_{t}^{\tr})$, where $\Abar_{t}$ is defined in~\eqref{eq:matrixA_t} and the covariance matrix $\Abar_{t}\SigmaWbar{}\Abar_{t}^{\tr}$ is symmetric and positive definite. Therefore, $\vecX_{t}$ can be expressed as an affine transformation, $\vecX_{t} = \Abar_{t}\Bbar\vecUbar + \Abar_{t}\matrixSqrt{\SigmaWbar{}}\nu$, of the standard random vector $\nu \sim \Pdist(\zero,\eye)$ under the condition that this mapping transforms the distribution of $\nu$ to that of $\vecWbar$. Consequently, the characteristic function of the random vector $\vecX_{t}$ from that of $\nu$,  according to Definition~\ref{definition:Characteristic_Function}, is
\begin{equation} \label{eq:Characteristic_Function}
\Expectation{\expS{j \innerS{f}{\vecX_{t}}}} = \expS{j\innerS{f}{\Abar_{t}\Bbar\vecUbar}} \characteristicFuncS{\matrixSqrt{\SigmaWbar{}}\Abar_{t}^{\tr}f}.
\end{equation}
Using the Fourier basis~\eqref{eq:Fourier_Basis}, it follows that the $(n+1)^\nth$ element of the mean vector $\MuPhi^{t}$ is given by
$$
\begin{aligned}
\Expectation{\phi_{n}(\vecX_{t})} =
\begin{cases}
1 & \!\!n = 0 \\
\, \binarySum{\ell \in \{-1,1\}}{\expS{j\innerS{\Abar_{t}^{\tr} \ell f_{n}}{\Bbar\vecUbar}}\characteristicFuncS{\matrixSqrt{\SigmaWbar{}} \Abar_{t}^{\tr} \ell f_{n}}} & \!\!n \geq 1.
\end{cases}
\end{aligned}
$$
In addition, the $(m+1,n+1)^\nth$ element of the covariance matrix $\SigmaPhi^{tt'}$ is expressed as $$
\SigmaPhiElements{mn}^{tt'} = \Expectation{\phi_{m}(\vecX_{t})\phi_{n}(\vecX_{t'})} - \Expectation{\phi_{m}(\vecX_{t})}\Expectation{\phi_{n}(\vecX_{t'})},
$$ 
which can also be derived from~\eqref{eq:Fourier_Basis} and~\eqref{eq:Characteristic_Function}. When at least one index is zero ($m = 0$ or $n = 0$), the element simplifies to
$$
\begin{aligned}
\SigmaPhiElements{00}^{tt'} & = \Expectation{\phi_{0}(\vecX_{t})\phi_{0}(\vecX_{t'})} - \Expectation{\phi_{0}(\vecX_{t})}\Expectation{\phi_{0}(\vecX_{t'})} = 0, \\
\SigmaPhiElements{m0}^{tt'} & = \Expectation{\phi_{m}(\vecX_{t})\phi_{0}(\vecX_{t'})} - \Expectation{\phi_{m}(\vecX_{t})}\Expectation{\phi_{0}(\vecX_{t'})} = 0, \\
\SigmaPhiElements{0n}^{tt'} & = \Expectation{\phi_{0}(\vecX_{t})\phi_{n}(\vecX_{t'})} - \Expectation{\phi_{0}(\vecX_{t})}\Expectation{\phi_{n}(\vecX_{t'})} = 0.
\end{aligned}
$$
Finally, for the case where both $m \geq 1$ and $n \geq 1$, the term $\phi_{m}(\vecX_{t})\phi_{n}(\vecX_{t'})$ is a summation of the following four terms:
$$
\phi_{m}(\vecX_{t})\phi_{n}(\vecX_{t'}) = \binarySum{\ell,\ell'\in\{-1,1\}}{\expS{j\innerS{\ell f_{m}}{\vecX_{t}}} \expS{j\innerS{\ell'\! f_{n}}{\vecX_{t'}}}}.
$$
Given that $\vecX_{t} = \Abar_{t}\Bbar\vecUbar + \Abar_{t}\matrixSqrt{\SigmaWbar{}}\nu$ and $\vecX_{t'} = \Abar_{t'}\Bbar\vecUbar + \Abar_{t'}\matrixSqrt{\SigmaWbar{}}\nu$, we can simplify $\phi_{m}(\vecX_{t})\phi_{n}(\vecX_{t'})$ to
$$
\phi_{m}(\vecX_{t})\phi_{n}(\vecX_{t'}) = \binarySum{\ell,\ell'\in\{-1,1\}}{
\expS{j\innerS{\Abar_{t}^{\tr}\ell f_{m}+\Abar_{t'}^{\tr}\ell'\! f_{n}}{\Bbar\vecUbar}}
\expS{j\innerS{\matrixSqrt{\SigmaWbar{}}\! (\Abar_{t}^{\tr}\ell f_{m}+\Abar_{t'}^{\tr}\ell'\! f_{n})}{\nu}}.
}
$$
Thus, for the case where $m,n \geq 1$, the $(m+1,n+1)^\nth$ element of the covariance matrix $\SigmaPhi^{tt'}$ is
$$
\SigmaPhiElements{mn}^{tt'} \!= - \MuPhiElements{m}^{t}\MuPhiElements{n}^{t'} + \binarySum{\ell,\ell'\in\{-1,1\}}{
\expS{j\innerS{\Abar_{t}^{\tr}\ell f_{m}+\Abar_{t'}^{\tr}\ell'\! f_{n}}{\Bbar\vecUbar}}
\characteristicFunc{\matrixSqrt{\SigmaWbar{}}\! (\Abar_{t}^{\tr}\ell f_{m}+\Abar_{t'}^{\tr}\ell'\! f_{n})},
} 
$$
which can be rewritten as~\eqref{eq:DBS_Covariance}, completing the proof.
\end{proof}
\begin{proof}[Proof of Proposition~\ref{proposition:Bayesian_MMSE_consistency}] 
\label{proof:Bayesian_MMSE_consistency}
Let $\Cost^{\star}_{\mathrm{B}}(t)$ be the optimal estimation error~\eqref{eq:MMSE_OptimalCost} at time $t$ given an output measurement trajectory with length $t+1$, and consider its equivalent representation in~\eqref{eq:MMSE_OptimalCost_InverseLemma}. From Corollary~\ref{corollary:Bayesian_Error_Reduction}, we derive 
\begin{equation} \label{eq:Decomposition}
\Cost^{\star}_{\mathrm{B}}(t) = \trace \big( \big( \SigmaTheta{}^{\inv}\!+\matrixS(t-1)+\Delta\matrixS(t) \big)^{\inv} \big),
\end{equation}
where $\matrixS(t-1) = \matrixPhibar(t-1) \matrixR^{\inv}\!(t-1)\matrixPhibar^{\tr}\!(t-1)$, $\Delta\matrixS(t) = \gamma^{\inv}\!(t) \vecS(t)\vecS^{\tr}\!(t)$, $\vecS(t) = \matrixPhibar(t-1) \matrixR^{\inv}\!(t-1)\vecC(t) - \MuPhi^{t}$,
\begin{equation} \label{eq:Decomposed_Matrices}
\begin{cases}
\gamma(t) = \vecR(t)-\vecC^{\tr}\!(t) \matrixR^{\inv}\!(t-1)\vecC(t) > 0, \\
\vecR(t) = \matrixM_{tt}+\sigmaV{t}^{2}, \\
\vecC(t) = [\matrixM_{0t}, \hdots, \matrixM_{(t-1)t}]^{\tr}\!.
\end{cases}
\end{equation}
Recursively applying the decomposition~\eqref{eq:Decomposition} from $t-1$ to $1$, we obtain
$$
\Cost^{\star}_{\mathrm{B}}(t) = \trace \big( \big( \SigmaTheta{}^{\inv}\!+\matrixS(0)+\Delta\matrixS(1)+\cdots+\Delta\matrixS(t) \big)^{\inv} \big) = \trace \big( \big( \SigmaTheta{}^{\inv}\!+\matrixS(0)+ \sumOp_{i=1}^{t}
\frac{1}{\gamma(i)} \vecS(i)\vecS^{\tr}\!(i) \big)^{\inv} \big).
$$
Since $\SigmaTheta{}^{\inv}\!\succeq \zeromx$ and $\matrixS(0) \succeq \zeromx$, $\Cost^{\star}_{\mathrm{B}}(t)$ tends to $0$ as $t$ goes to $\infty$, if and only if the smallest eigenvalue of the information matrix diverges to infinity, i.e.,
$$
\begin{aligned}
\limOp_{t \to \infty} \Cost^{\star}_{\mathrm{B}}(t) = 0 & \iff 
\limOp_{t \to \infty} \lambda_{\max}\Big(\big( \SigmaTheta{}^{\inv}\!+\matrixS(0)+ \sumOp_{i=1}^{t}
\frac{1}{\gamma(i)} \vecS(i)\vecS^{\tr}\!(i) \big)^{\inv}\Big) = 0,\\
& \iff \limOp_{t \to \infty} \lambda_{\min}\big( \SigmaTheta{}^{\inv}\!+\matrixS(0)+ \sumOp_{i=1}^{t}
\frac{1}{\gamma(i)} \vecS(i)\vecS^{\tr}\!(i) \big) = \infty.
\end{aligned}
$$
Equivalently, the eigenvalues diverge if and only if for every unit vector $\hat{v} \in \Real^{(\Ntheta+1)}$, i.e., $\norm{\hat{v}} = 1$, the above series diverges to infinity, i.e.,
$$
\limOp_{t \to \infty} \lambda_{\min}\big( \SigmaTheta{}^{\inv}\!+\matrixS(0)+ \sumOp_{i=1}^{t}
\frac{1}{\gamma(i)} \vecS(i)\vecS^{\tr}\!(i) \big) = \infty \iff 
\forall \hat{v} \in \Real^{(\Ntheta+1)}, \, \norm{\hat{v}} = 1, \,
\limOp_{t \to \infty} \sumOp_{i=1}^{t}
\frac{1}{\gamma(i)} (\hat{v}^{\tr}\!\vecS(i))^{2} = \infty.
$$
Hence, the necessary and sufficient condition for the convergence of the estimator to the true value is
\begin{equation} \label{eq:Converge_Condition}
\limOp_{t \to \infty} \Cost^{\star}_{\mathrm{B}}(t) = 0 \iff \limOp_{t \to \infty} \inf_{\|v\| = 1} \Big( \sumOp_{i=1}^{t}
\frac{1}{\gamma(i)} (\hat{v}^{\tr}\!\vecS(i))^{2} \Big) = \infty.
\end{equation}
For the sake of contradiction, let us choose the vector $\one_{\{n\}}$ as the $n^{\nth}$ standard basis vector, i.e.,
$$
\begin{aligned}
\one_{\{n\}} & = [0, \hdots, 0, 1, 0, \hdots, 0]^{\tr}\!, \quad n \geq 1. \\[-10.7pt]
& \phantom{[0, \hdots, 0,.. 1}\hspace{3.4pt} \text{\scalebox{0.9}{$\uparrow$}} \\[-11pt]
& \phantom{[0, \hdots, 0,.. 1}\hspace{3.3pt} \text{\scalebox{0.9}{$n^{\nth}$}} \\[-7pt]
\end{aligned}
$$
Substituting this unit vector into the convergence condition~\eqref{eq:Converge_Condition}, we have
$$
\frac{1}{\gamma(i)} (\one_{\{n\}}^{\tr} \vecS(i))^{2} = \frac{1}{\gamma(i)} (\one_{\{n\}}^{\tr} \matrixPhibar(i-1) \matrixR^{\inv}\!(i-1)\vecC(i) - \one_{\{n\}}^{\tr} \MuPhi^{i})^{2} = \frac{1}{\gamma(i)} (\sumOp_{j=0}^{i-1}\frac{1}{\gamma(j)}\MuPhiElements{n}^{j}\matrixM_{ji} - \MuPhiElements{n}^{i})^{2},
$$
and
\begin{equation} \label{eq:error_bound}
\limOp_{t \to \infty} \sumOp_{i=1}^{t}
\frac{1}{\gamma(i)} (\one_{\{n\}}^{\tr} \vecS(i))^{2} = \sumOp_{i=1}^{\infty} \frac{1}{\gamma(i)} (\sumOp_{j=0}^{i-1}\frac{1}{\gamma(j)}\MuPhiElements{n}^{j}\matrixM_{ji} - \MuPhiElements{n}^{i})^{2},
\end{equation}
where $\gamma(0) = \matrixM_{00}+\sigmaV{0}^{2}$. From the structure of $\MuPhiElements{n}^{i}$ and $\SigmaPhiElements{nm}^{ji}$ for the Fourier basis with Gaussian process noise in~\eqref{eq:Gaussian_Characteristic}, i.e.,
\begin{align*}
\left\{\!\!\!
\begin{array}{rl}
\MuPhiElements{n}^{i} & \!= \chi_{n}^{i} \exp{\!-\frac{1}{2}f_{n}^{\tr}\Abar_{i} \SigmaWbar{} \Abar_{i}^{\tr}f_{n}}, \mkern9mu \qquad -2 \leq  \chi_{n}^{i} \leq 2, \\
\SigmaPhiElements{mn}^{ji} & \!= \xi^{ji}_{mn} \exp{\!-\frac{1}{2} f_{n}^{\tr}\Abar_{i}\SigmaWbar{}\Abar_{i}^{\tr}f_{n}}, \qquad -2 \leq \xi^{ji}_{mn} \leq 2, \\
\chi_{n}^{i} & \!= \binarySum{\ell \in \{-1,1\}}{\expS{j\innerS{\Abar_{i}^{\tr} \ell f_{n}}{\Bbar\vecUbar}}}, \\
\xi^{ji}_{mn} & \!= \binarySum{\ell,\ell'\in\{-1,1\}}{
\expS{j\innerS{\Abar_{j}^{\tr}\ell f_{m}+\Abar_{i}^{\tr}\ell'\! f_{n}}{\Bbar\vecUbar}} 
\big(\exp{\!-\!\ell f_{m}^{\tr}\Abar_{j}\SigmaWbar{}\Abar_{i}^{\tr}\ell'\! f_{n}}-1 \big)
\exp{\!-\frac{1}{2}  f_{m}^{\tr}\Abar_{j}\SigmaWbar{}\Abar_{j}^{\tr}f_{m}},}
\end{array}\right. 
\end{align*}
we have the following properties:
\begin{enumerate} [label=(\roman*)]
\item Bounded coefficients: $\MuPhiElements{n}^{i} < \infty$, $\MuPhiElements{n}^{j} < \infty$, $\SigmaPhi^{ji} < \infty$ for all values of $i$ and $j$. Also, from~\eqref{eq:Decomposed_Matrices}, 
$$
\matrixM_{ji}=\trace \big(\SigmaPhi^{ji}(\SigmaTheta{} + \MuTheta{}\MuTheta{}^{\tr}) \big) < \infty, \quad
0 < \gamma(i) \leq \matrixM_{ii}+\sigmaV{i}^{2} < \infty, \quad 0 < \gamma(j) \leq \matrixM_{jj}+\sigmaV{j}^{2} < \infty.
$$
\item Impact of stochastic instability: For the special case of  marginally stable or unstable LTI system, i.e., $\lambda_{\max}(\A) \geq 1$,
$$
\tilde{v}^{\tr}\!(\Abar_{i}\SigmaWbar{}\Abar_{i}^{\tr})\tilde{v} = \sumOp_{k=0}^{i} \tilde{v}^{\tr}\!\A^{\!k}\SigmaW{(i-k)}\A^{\!k^{\tr}}\!\tilde{v} 
= \sumOp_{k=0}^{i} |\lambda_{q}(\A)|^{2k} \tilde{v}^{\tr}\SigmaW{(i-k)}\tilde{v},
$$
where $\tilde{v}$ is the eigenvector of matrix $\A$ corresponding to $\lambda_{\max}(\A)$ and $\SigmaW{0} = \SigmaX{0}$. It is evident from the equality that $\Abar_{i}\SigmaWbar{}\Abar_{i}^{\tr}$ grows unbounded and $\expS{\!-\frac{1}{2}f_{n}^{\tr}\Abar_{i} \SigmaWbar{} \Abar_{i}^{\tr}f_{n}}$ decays to zero.
\item Decay of terms: The elements $\MuPhiElements{n}^{j}$, $\matrixM_{ji}$, and $\MuPhiElements{n}^{i}$~in~\eqref{eq:error_bound} decay exponentially to zero due to $\expS{\!-\frac{1}{2}f_{n}^{\tr}\Abar_{i} \SigmaWbar{} \Abar_{i}^{\tr}f_{n}}$.
From~\eqref{eq:Decomposed_Matrices}, one can also observe that $\vecC(i)$ decays exponentially to zero, while $\vecR(i)$ exponentially converges to $\sigmaV{i}^{2}$. Consequently, $\gamma(i)$ and $\gamma(j)$ also exponentially converge to $\sigmaV{i}^{2}$ and $\sigmaV{j}^{2}$, respectively.
\end{enumerate}
Thus, the numerator of the series~\eqref{eq:error_bound} decays exponentially as $\mathcal{O} \big( \expS{\!-\frac{1}{2}f_{n}^{\tr}\Abar_{i} \SigmaWbar{} \Abar_{i}^{\tr}f_{n}} \big)$, and therefore, the summation converges to a bounded value, i.e.,
$$
\limOp_{t \to \infty} \sumOp_{i=1}^{t} \frac{1}{\gamma(i)} (\one_{\{n\}}^{\tr}\vecS(i))^{2} < \infty.
$$
Consequently, there exists a unit vector $\hat{v} = \one_{\{n\}}$ and $n\geq1$, such that the series converges, violating the condition~\eqref{eq:Converge_Condition}. Therefore, $\lim_{t \to \infty} \Cost^{\star}_{\mathrm{B}}(t) > 0$, proving the Bayesian MMSE affine estimator is not consistent for Fourier bases and stochastically unstable LTI systems with Gaussian process noise~$\vecWbar$.
\end{proof}
\begin{proof}[Proof of Proposition~\ref{proposition:Multi_Batch_Bayesian_MMSE_consistency}]
\label{proof:Multi_Batch_Bayesian_MMSE_consistency}
Let $\Cost^{\star}_{\mathrm{B}}(\batch, \Trajectory^{\batch})$ denote the optimal Bayesian MMSE estimation error~\eqref{eq:MMSE_OptimalCost}, where $\batch$ is the number of statistically independent output trajectories, each with the length of $(\Trajectory_{i}+1)$, namely,  $\vecYbar_{i} = [\vecY_{0}, \hdots, \vecY_{\Trajectory_{i}}]^{\tr}\!$, $i \in \{1,...,\tau\}$. Thus, a total of $\Trajectory^{\batch}\!= (\Trajectory_{1}+1) + \cdots + (\Trajectory_{\batch}+1)$ measurements are used. Since the error decreases monotonically with additional dependent measurements, as stated in Corollary~\ref{corollary:Bayesian_Error_Reduction}, one can infer that
\begin{equation} \label{eq:Data_Length_Cost_Inequality}
\Cost^{\star}_{\mathrm{B}}(\batch, \Trajectory^{\batch}) \leq \Cost^{\star}_{\mathrm{B}}(\batch, \batch),
\end{equation}
where $\Cost^{\star}_{\mathrm{B}}(\batch, \batch)$ represents the estimation error using $\batch$ statistically independent trajectories each with single measurements, $\Trajectory^{\batch}\!= \tau$, where each trajectory $i$ contains a single independent measurement, i.e., $\vecYbar_{i} = \vecY_{0}$. Using the equivalent representation of the optimal MMSE error~\eqref{eq:MMSE_OptimalCost_InverseLemma}, we have
$$
\Cost^{\star}_{\mathrm{B}}(\batch, \batch) = \trace \big( \big(\SigmaTheta{}^{\inv}\!+\matrixPhibar(\batch)\matrixR^{\inv}\!(\batch)\matrixPhibar^{\tr}\!(\batch)\big)^{\inv} \big),
$$
where $\matrixR(\batch) \coloneqq \matrixM(\batch)+\SigmaVbar{}(\batch)$. Consider the matrix definition and decompositions $\matrixS \coloneqq \matrixPhibar(\batch)\matrixR^{\inv}\!(\batch)\matrixPhibar^{\tr}\!(\batch)$,
$$
\begin{aligned}
& \matrixPhibar(\batch) = [\MuPhi^{0}(1), \hdots, \MuPhi^{0}(\batch)], \qquad
\matrixR(\batch) = \diag \big( \vecR(1), \hdots, \vecR(\batch) \big),
\end{aligned}
$$
with $\vecR(i) = \matrixM_{00}(i)+\sigmaV{0}^{2}(i) > 0$ for batch $i$, and 
$\matrixM_{00}(i) = \trace \big(\SigmaPhi^{00}(i)(\SigmaTheta{} + \MuTheta{}\MuTheta{}^{\tr}) \big)$.
The statistical independence between batches ensures $\matrixR(\batch)$ is diagonal, leading to 
$$
\matrixS = \sumOp_{i=1}^{\batch} \frac{1}{\vecR(i)}\MuPhi^{0}(i)\MuPhi^{0^{\tr}}\!(i).
$$
Following the reasoning in the proof of Proposition~\ref{proposition:Bayesian_MMSE_consistency}, the necessary and sufficient condition for the convergence of the MMSE estimator with independent and identically distributed (i.i.d.) data is
$$
\limOp_{\batch \to \infty} \Cost^{\star}_{\mathrm{B}}(\batch, \batch) = 0 \iff  
\limOp_{\batch \to \infty}
\lambda_{\min} \Big( \sumOp_{i=1}^{\batch} \frac{1}{\vecR(i)}\MuPhi^{0}(i)\MuPhi^{0^{\tr}}\!(i)
\Big) = \infty.
$$
From the inequality~\eqref{eq:Data_Length_Cost_Inequality}, the above provides a sufficient condition for $\Cost^{\star}_{\mathrm{B}}(\batch, \Trajectory^{\batch})$ to converge to $0$.
\end{proof}
\section{Active Learning} \label{sec:Active_Learning} 
Active learning seeks to develop input signals that maximize information gain, thereby reducing estimation error. While our proposed Bayesian MMSE affine estimator in Theorem~\ref{theorem:MMSE_LinearEstimator} is optimal among all affine estimators, it can be coupled with an optimal input signal to improve estimation performance further. To this end, our approach leverages the analytical expression of the estimation error~\eqref{eq:MMSE_OptimalCost}, distinguishing it from most active learning methods that rely on maximizing information gain as a proxy for the estimation error, which is typically unavailable. Consequently, the optimal inputs can be determined independently of measurements, either a-priori or in real-time, by solving the following optimization problem:
\begin{equation} \label{eq:Active_Learning_Optimization}
\vecUbar^{\star} \in \arg\!\min_{\vecUbar \in \Ubb} \autoExpectation{\big\lVert \vecTheta - \hat{\vecTheta}_{\mathrm{B}}(\vecYbar) \big\rVert^{2}} = \arg\!\min_{\vecUbar \in \Ubb} \Cost^{\star}_{\mathrm{B}}(\vecUbar),
\end{equation}
where $\Ubb$ represents the input space, which may impose physical constraints on feasible inputs for estimation. Since only the second term of $\Cost^{\star}_{\mathrm{B}}$ in~\eqref{eq:MMSE_OptimalCost} depends on $\vecUbar$ and is negative, the optimization problem~\eqref{eq:Active_Learning_Optimization} is equivalent to maximizing the second term of~\eqref{eq:MMSE_OptimalCost}. Nevertheless, we present this problem in its minimization form and solve it using an iterative first-order method, such as steepest descent or projected steepest descent when constraints are involved~\cite{bertsekas2015convex}, while leveraging an analytical expression for its gradient. It should be noted that~\eqref{eq:Active_Learning_Optimization} is potentially non-convex due to how $\vecUbar$ influences matrices $\matrixPhibar$ and $\matrixM$. The iterative update rule for projected gradient descent is given by
\begin{equation} \label{eq:Projected_Steepest_Descent}
\vecUbar^{\iteration+1} = \setProject{\Ubb} \big[ \vecUbar^{\iteration} - \alpha_{\iteration} \gradient_{\!\vecUbar} \Cost^{\star}_{\mathrm{B}}(\vecUbar^{\iteration}) \big], \qquad
\gradient_{\!\vecUbar} \Cost^{\star}_{\mathrm{B}} = \begin{bmatrix}
\dfrac{\partial \Cost^{\star}_{\mathrm{B}}}{\partial \vecUbar_{1}} \fracComma \cdots \fracComma \dfrac{\partial \Cost^{\star}_{\mathrm{B}}}{\partial \vecUbar_{(\nX+\Trajectory\nU)}}
\end{bmatrix}^{\tr}\!\! \fracComma
\end{equation}
where $\iteration$ represents the current iteration step, $\setProject{\Ubb}[\cdot]$ denotes the projection operator that maps the argument onto $\Ubb$, $\alpha_{\iteration}$ is a positive stepsize, and $\gradient_{\vecUbar} \Cost^{\star}_{\mathrm{B}}(\vecUbar^{\iteration})$ is the gradient of the cost function evaluated at $\vecUbar^{\iteration}$. Various algorithms exist for selecting the stepsize $\alpha_{\iteration}$, including the standard approach of diminishing stepsize rules~\cite[p.~69]{bertsekas2015convex}. Furthermore, since the explicit description of $\Cost^{\star}_{\mathrm{B}}$ is available in~\eqref{eq:MMSE_OptimalCost}~(aka. zero-order information), more sophisticated methods such as exact line search can also be employed; we refer interested readers to~\cite[Ch.~2]{bertsekas2015convex} for further details. It is worth noting that these stepsize rules may require parameter tuning (e.g., a constant in diminishing stepsize methods) or involve computational overhead in line search techniques. To circumvent these possible limitations, one can also utilize the recent, easy-to-implement, adaptive stepsize from~\cite{malitsky20adaptive} defined as
$$
\alpha_{\iteration} \!=\! \min \biggl\{\!\sqrt{1+\beta_{\iteration-1}}\alpha_{\iteration-1} \fracComma \mkern2mu \frac{\norm{\vecUbar^{\iteration}-\vecUbar^{\iteration-1}}}{2\norm{\gradient_{\!\vecUbar} \Cost^{\star}_{\mathrm{B}}(\vecUbar^{\iteration}) - \gradient_{\!\vecUbar} \Cost^{\star}_{\mathrm{B}}(\vecUbar^{\iteration-1})}} \!\biggr\} \fracComma \qquad 
\beta_{\iteration} \!=\! \frac{\alpha_{\iteration}}{\alpha_{\iteration-1}} \fracComma \qquad k \geq 1,
$$
with initial conditions $\beta_{0} = \infty$ and $\alpha_{0} = 10^{-10}$. Finally, the gradient of the Bayesian MMSE estimation error~\eqref{eq:MMSE_OptimalCost} with respect to each $\vecUbar_{i}$ of the input vector is derived by applying the chain rule to $\Cost^{\star}_{\mathrm{B}}(\vecUbar)$, resulting in 
\begin{equation} \label{eq:Individual_Partial_Derivative} 
\dfrac{\partial \Cost^{\star}_{\mathrm{B}}}{\partial \vecUbar_{i}} 
= \trace \big( \Psi^{\star\tr}\!\Psi^{\star} \matrixGradient{\vecUbar_{i}}\matrixM + 2\Psi^{\star\tr}\! (\Psi^{\star}\matrixPhibar^{\tr} - \eye)\SigmaTheta{} \matrixGradient{\vecUbar_{i}}\matrixPhibar 
\big),
\end{equation}
where $\Psi^{\star}$ is defined in~\eqref{eq:MMSE_LinearEstimator}. The partial derivative terms $\matrixGradient{\vecUbar_{i}}\matrixM$ and $\matrixGradient{\vecUbar_{i}}\matrixPhibar$ depend on the gradient of the respective DBS as follows:
\begin{itemize}
\item $\matrixGradient{\vecUbar_{i}}\matrixPhibar = [\matrixGradient{\vecUbar_{i}}\MuPhi^{0}, \hdots, \matrixGradient{\vecUbar_{i}}\MuPhi^{\Trajectory}]$,
\item the $(t+1,t'+1)^{\nth}$ element of $\matrixGradient{\vecUbar_{i}}\matrixM$ is $\matrixGradient{\vecUbar_{i}}\matrixM_{tt'} = \trace \big((\SigmaTheta{} + \MuTheta{}\MuTheta{}^{\tr}) \matrixGradient{\vecUbar_{i}}\SigmaPhi^{tt'} \big)$,
\item the $(n+1)^{\nth}$ element of $\matrixGradient{\vecUbar_{i}}\MuPhi^{t}$ and $(m+1,n+1)^\nth$ element of $\matrixGradient{\vecUbar_{i}}\SigmaPhi^{tt'}\!$ are
$$
\matrixGradient{\vecUbar_{i}}\MuPhiElements{n}^{t} \!\!=\! \dfrac{\partial \Expectation{\phi_{n}(\vecX_{t})}}{\partial \vecUbar_{i}} \fracComma \qquad 
\matrixGradient{\vecUbar_{i}}\SigmaPhiElements{mn}^{tt'} \!\!=\! \dfrac{\partial \Expectation{\phi_{m}(\vecX_{t'})\phi_{n}^{\tr}(\vecX_{t})}}{\partial \vecUbar_{i}} - \matrixGradient{\vecUbar_{i}}\MuPhiElements{m}^{t'}\MuPhiElements{n}^{t} \!\!- 
\MuPhiElements{m}^{t'}\matrixGradient{\vecUbar_{i}}\MuPhiElements{n}^{t}\!.
$$
\end{itemize}
For Fourier basis functions~\eqref{eq:Fourier_Basis}, with the explicit expressions of the DBS $(\MuPhi^{t}, \SigmaPhi^{tt'})$ provided in Lemma~\ref{lemma:Fourier_basis_Explicit_DBS}, the partial derivatives $\matrixGradient{\vecUbar_{i}}\MuPhi^{t}$ and $\matrixGradient{\vecUbar_{i}}\SigmaPhi^{tt'}\!$, are the following DBS gradients:
\begin{enumerate} [label=(\roman*)]
\begin{subequations} \label{eq:DBS_Gradient}
\item \textbf{DBS mean gradient}: $\matrixGradient{\vecUbar_{i}}\MuPhiElements{0}^{t} \!= 0$, and for $n \geq 1$,
\begin{equation}
\label{eq:Partial_Derivative_MuPhi_Elements}
\matrixGradient{\vecUbar_{i}}\MuPhiElements{n}^{t} \!\!= \binarySum{\ell \in \{-1,1\}}{j\ell f_{n}^{\tr}\Abar_{t}\Bbar \one_{\{i\}} \expS{j\innerS{\Abar_{t}^{\tr} \ell f_{n}}{\Bbar\vecUbar}}\characteristicFuncS{\matrixSqrt{\SigmaWbar{}} \Abar_{t}^{\tr} \ell f_{n}}},
\end{equation}
\item \textbf{DBS covariance gradient}: $\matrixGradient{\vecUbar_{i}}\SigmaPhiElements{mn}^{tt'} \!= 0$ if $m = 0$ or $n = 0$. For $m, n \geq 1$, 
\begin{equation} \label{eq:Partial_Derivative_SigmaPhi_Elements}
\matrixGradient{\vecUbar_{i}}\SigmaPhiElements{mn}^{tt'} \!\!= \binarySum{\ell,\ell'\in\{-1,1\}}{\begin{multlined}[t]
j\big(\ell f_{m}^{\tr}\Abar_{t}+\ell'\! f_{n}^{\tr}\Abar_{t'}\big)\Bbar \one_{\{i\}} 
\expS{j\innerS{\Abar_{t}^{\tr}\ell f_{m}+\Abar_{t'}^{\tr}\ell'\! f_{n}}{\Bbar\vecUbar}} \\
\qquad \times \big( \characteristicFunc{\matrixSqrt{\SigmaWbar{}}\! (\Abar_{t}^{\tr}\ell f_{m}+\Abar_{t'}^{\tr}\ell'\! f_{n})}
- \characteristicFuncS{\matrixSqrt{\SigmaWbar{}} \Abar_{t}^{\tr} \ell f_{m}} \characteristicFuncS{\matrixSqrt{\SigmaWbar{}} \Abar_{t'}^{\tr} \ell'\! f_{n}}\big),
\end{multlined}
}
\end{equation}
\end{subequations}
\end{enumerate}
where $\Abar_{t}$ and $\Abar_{t'}$ are defined analogously to~\eqref{eq:matrixA_t}, and $\one_{\{i\}}$ is the single-entry vector with $1$ at index $i$ and zero elsewhere, i.e.,
$$
\begin{aligned}
\one_{\{i\}} & = [0, \hdots, 0, 1, 0, \hdots, 0]^{\tr}\!. \\[-10.7pt]
& \phantom{[0, \hdots, 0,.. 1}\hspace{3.4pt} \text{\scalebox{0.9}{$\uparrow$}} \\[-9.1pt]
& \phantom{[0, \hdots, 0,.. 1}\hspace{3.5pt} \text{\scalebox{0.9}{$i^{\nth}$}} \\[-7pt]
\end{aligned}
$$
In the case of Gaussian noise, it is sufficient to substitute the characteristic functions in the derived terms of~\eqref{eq:DBS_Gradient} with their corresponding expressions in~\eqref{eq:Gaussian_Characteristic}. The computational complexity of computing the gradient~\eqref{eq:Individual_Partial_Derivative} for each element of input per iteration of the first-order method is equivalent to that of the optimal Bayesian MMSE affine estimator, which scales as $\Ocal\big(\Trajectory^{3}\big)$ when $\Ntheta \ll \Trajectory$. In the following section, we numerically demonstrate the reduction in estimation error achieved by applying this active learning technique.
\section{Numerical Experiments} \label{sec:Numerical_Experiments} 
In this section, we evaluate the performance of four estimators through numerical examples: two \emph{approximate} regularized least-squares (RLS) linear estimators (DLS and MLS) from Section~\ref{subsec:Least_Square}, the optimal Bayesian MMSE affine estimator (BMS) from Section~\ref{subsec:Bayesian_Estimation}, and its integration with active learning (BAL) from Section~\ref{sec:Active_Learning}. To ease the reproducibility of these experiments, we provide our MATLAB library at \href{https://github.com/sasanvakili/Bayesian4Wiener}{https://github.com/sasanvakili/Bayesian4Wiener}.

Using a marginally stable dynamical system with a true function in Fourier subspace, we observe the effects of increasing process noise uncertainty and demonstrate the inherent inconsistency discussed in Proposition~\ref{proposition:Bayesian_MMSE_consistency}. We compare estimators using the squared error criterion, $||\vecTheta - \hat{\vecTheta}(\vecYbar)||^{2}$, over $10,\!000$ simulations, employing identical realizations of $\vecWbar$ and $\vecVbar$ across all estimators for fair comparison.
In the resulting plots, dashed lines represent the analytically computed mean squared error, $\Expectation{||\vecTheta - \hat{\vecTheta}(\vecYbar)||^{2}}$, for each method. Non-histogram plots display shaded areas representing the $20^{\nth}$ to $80^{\nth}$ percentile range of squared error, while histograms show probability densities.
\paragraph{\bf Experiment setup}
Our experiments are based on time-invariant Gaussian noise, $\vecV_t \sim \NormalDist( \zero, \sigma^{2}_{\vecV}\eye)$, $\vecW_{t+1}\sim \NormalDist(\zero, \sigma^{2}_{\vecW}\eye )$, and $\vecX_0 \sim \NormalDist(\MuX{0}, \sigma^{2}_{\vecX_{0}}\eye)$, with $\sigma^{2}_{\vecX_{0}} = \sigma^{2}_{\vecW}$. To examine the impact of process noise, we use three incremental variances, $\sigma^{2}_{\vecW} = \{0,\, 0.001,\, 0.01\}$, while maintaining a consistent measurement noise variance of $\sigmaV{}^{2} = 0.01$ across all experiments. Our experimental design involves generating $100$ independent samples each of $\vecTheta$, $\vecWbar$, and $\vecVbar$ for various trajectory lengths $\Trajectory$, resulting in $10,\!000$ experiments. The dynamical system under study is a linear time-invariant kinematic model representing a robot moving in two dimensions. In this model, the system states $\vecX_{t} \in \Real^{2}$ correspond to the robot's position, while the inputs $\vecU_{t} \in \Real^{2}$ represent velocities in each dimension. The system dynamics are defined by $\A_{t} = \eye$ and $\B_{t} = \Delta t \eye$, with a sampling time $\Delta t = 0.1$. We use $\MuX{0}=[3.2, 2.8]^{\tr}$ and $\vecU_{t} = 4.5\!\sum_{\upsilon \in \Upsilon} [\cos(\upsilon t),\ \sin(\upsilon t)]^{\tr}$, $\Upsilon = \{3, 5, 10, 20, 100\}$, for DLS, MLS, and BMS experiments, as well as for initializing the projected steepest descent algorithm of BAL. Each dimension of $\vecU_{t}$ is constrained within $[-200, 200]$, representing the robot's achievable velocity range. The BAL estimator derives $\vecUbar^{\star}$ within this input range as described in Section~\ref{sec:Active_Learning}. To ensure consistent initial states across all estimators, BAL does not optimize the first element of $\vecUbar$, namely $\MuX{0}$. Following the Fourier basis representation in~\eqref{eq:Fourier_Basis}, we employ $11$ unknown parameters $\vecTheta_{n}$, i.e., $n=0,\hdots,10$, with known frequency vectors $f_{n} \in \Real^{2}$. Specifically, $f_0 = [0,0]^{\tr}$, $f_n = [n\frac{2\pi}{10},0]^{\tr}$ for $n=\{1,2,3\}$, and $f_n = [(n-7)\frac{2\pi}{10},\frac{2\pi}{6}]^{\tr}$ for $n=\{4,...,10\}$. The prior distributions of the unknown parameters follow a uniform distribution $\UniformDist(2,8)$, implying $\MuTheta{n} = 5$ and $\sigmaTheta{n}^{2} = 3$, from which the true parameter values are drawn.
\paragraph{\bf Benchmark~1: optimal Bayesian vs.~RLS expected error}
Our first numerical benchmark involves $10,\!000$ simulations for a trajectory of $\Trajectory=100$, i.e., $101$ total measurements, to tune the hyperparameter $\lambda$ of DLS and MLS for three incremental process noise scenarios. The error of BMS remains constant across $\lambda$ values, as it does not depend on this hyperparameter. Figure~\ref{Fig:LSHyperParam_Error} illustrates that the squared error of DLS and MLS deviates increasingly from BMS as the process noise variance increases. When $\sigma^{2}_{\vecW} = 0$, as shown in the left plot of Figure~\ref{Fig:LSHyperParam_Error}, all three estimators perform comparably for small $\lambda$ values. However, the middle and right plots of Figure~\ref{Fig:LSHyperParam_Error} demonstrate that DLS and MLS significantly underperform compared to BMS when $\sigma^{2}_{\vecW} = \{0.001,\, 0.01\}$.
\begin{figure}[htbp]
\centering 
\includegraphics[width=0.3285\linewidth]{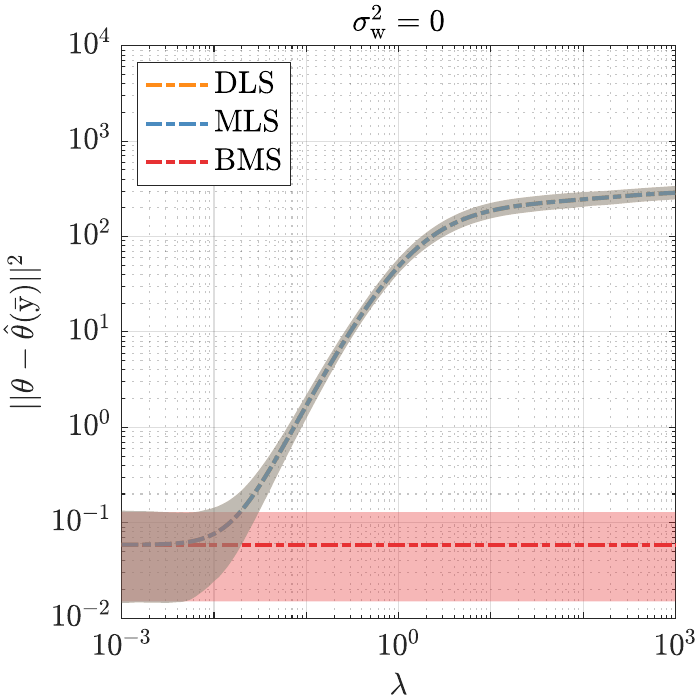} \hfill
\includegraphics[width=0.3285\linewidth]{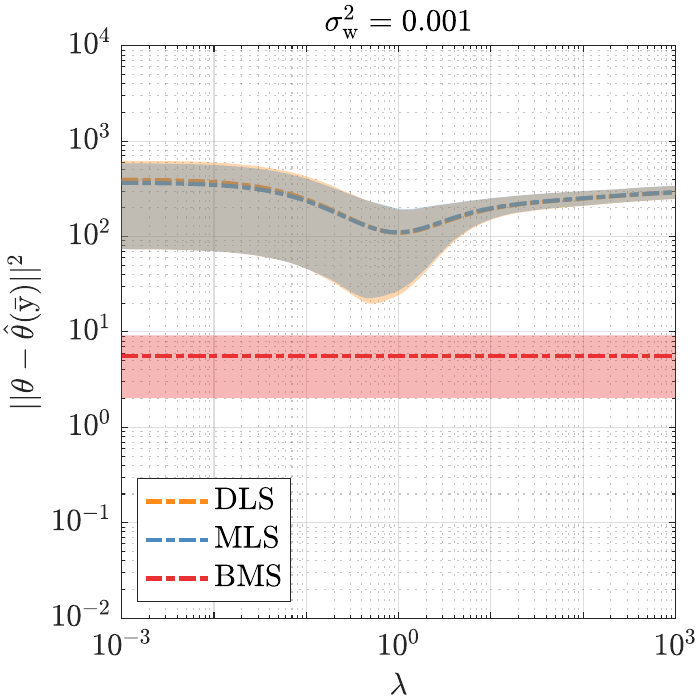} \hfill
\includegraphics[width=0.3285\linewidth]{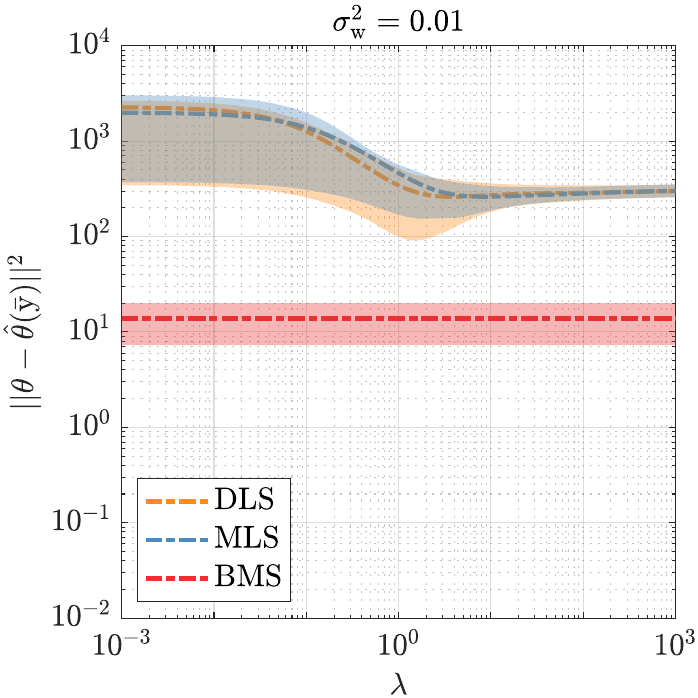}
\caption{Squared errors of DLS, MLS, and BMS for different $\lambda$ values}
\label{Fig:LSHyperParam_Error}
\end{figure}
\paragraph{\bf Benchmark~2: optimal Bayesian vs.~optimal RLS error histogram}
After tuning the hyperparameter $\lambda$ from Benchmark~1 and determining the optimal hyperparameter $\lambda^{\star}$ for MLS in each process noise case for a trajectory length of $\Trajectory=100$, we compare the squared error difference between methods. The estimates are denoted as follows: MLS using $\lambda^{\star}$ as $\hat{\vecTheta}_{\mathrm{MLS}}(\vecYbar, \lambda^{\star})$, BMS as $\hat{\vecTheta}_{\mathrm{BMS}}(\vecYbar)$, and BAL as $\hat{\vecTheta}_{\mathrm{BAL}}(\vecYbar, \vecUbar^{\star})$, where BAL utilizes optimal inputs $\vecUbar^{\star}$. Figure~\ref{Fig:ErrorDiff_Histogram} presents the probability density histogram of squared error differences between these pairs for $101$ total measurements and across $10,\!000$ simulations. Each histogram's area sums to unity, representing the probability density function, with dashed lines indicating the analytically computed mean of the differences. The red, blue, and orange distributions correspond to process noise variances of $\sigma^{2}_{\vecW} = 0$, $0.001$, and $0.01$, respectively. As observed in the left plots, MLS shows less squared error than BMS in only $1.62\%$ and $0.03\%$ of cases when $\sigma^{2}_{\vecW} = 0.001$ and $0.01$, respectively. The middle plot demonstrates that MLS never outperforms BAL. These results indicate that BMS and BAL almost surely outperform MLS when process noise exists. Without process noise, BMS and MLS yield similar results, as evident by the smaller red histogram centred around zero on the left plot. The right plot shows that BAL with $\vecUbar^{\star}$ consistently provides lower error than BMS with the input indicated in Experiment Setup when $\sigma^{2}_{\vecW} = 0$ or $0.001$, while BMS outperforms BAL in $1.7\%$ of cases when $\sigma^{2}_{\vecW} = 0.01$.
\begin{figure}[htbp]
\centering 
\includegraphics[width=0.3285\linewidth]{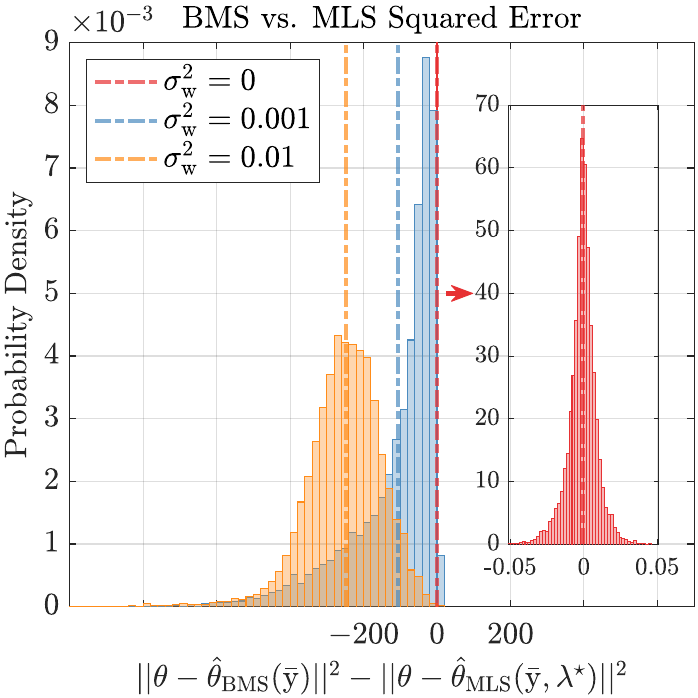} \hfill
\includegraphics[width=0.3285\linewidth]{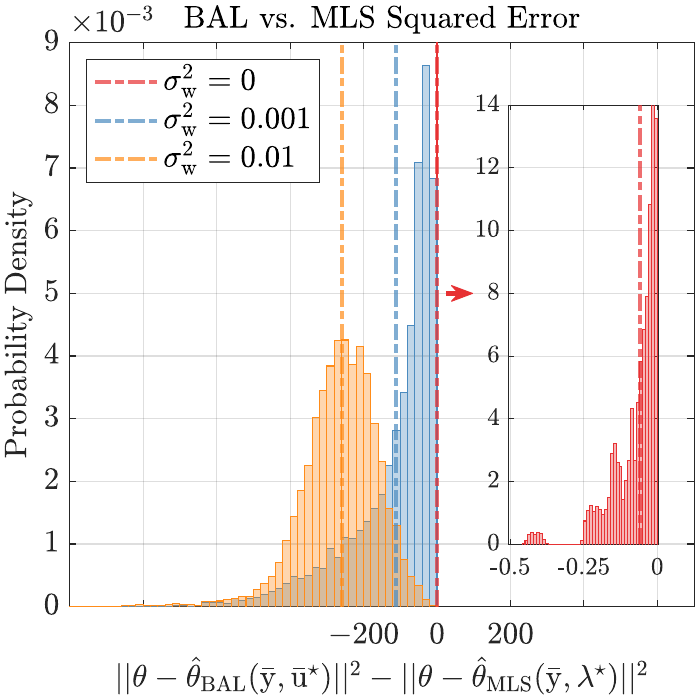} \hfill
\includegraphics[width=0.3285\linewidth]{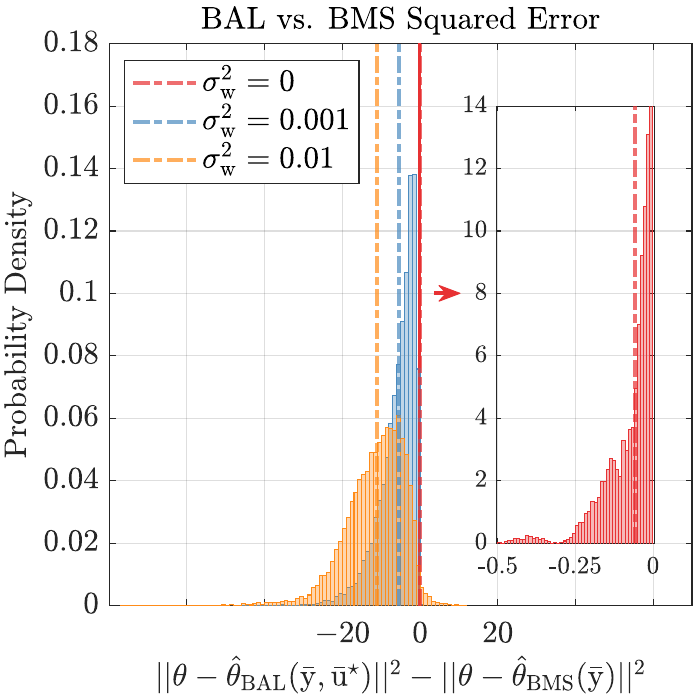}
\caption{Distribution of squared error differences between pairs of estimators}
\label{Fig:ErrorDiff_Histogram}
\end{figure}
\paragraph{\bf Benchmark~3: impact of trajectory length and process noise} 
Building on our previous benchmarks, we extend our analysis to $10,\!000$ simulations, each featuring trajectories of varying lengths, $\Trajectory \in \{0,\, 4,\, 10,\, 13,\, 16,\, 20,\, 25,\, 32,\, 40,\, 50,\, 63,\, 79,\, 100\}$. Figure~\ref{Fig:MeasurementLen_Error_OptInput} compares the squared errors, shown as shaded areas, and analytically computed mean squared errors, represented by dashed lines, for DLS, MLS, BMS, and BAL across three process noise variances. For DLS and MLS, we employ the optimal hyperparameter $\lambda^{\star}$, while BAL utilizes optimized input $\vecUbar^{\star}$, derived separately for each trajectory length. The results confirm that input optimization significantly reduces squared error when the measurement count equals the number of unknown parameters. In the presence of process noise, i.e., $\sigma^{2}_{\vecW} = \{0.001,\, 0.01\}$, DLS and MLS diverge as they fail to account for system states drift due to accumulated process noise. In contrast, Bayesian methods maintain robustness by precisely calculating the covariance matrices. Notably, the estimation error reduction of Bayesian estimators slows considerably with increasing statistically dependent measurements. This observation confirms the inconsistency described in Proposition~\ref{proposition:Bayesian_MMSE_consistency}, which arises from the marginal stability of the chosen dynamical system.
\begin{figure}[htbp]
\centering 
\includegraphics[width=0.3285\linewidth]{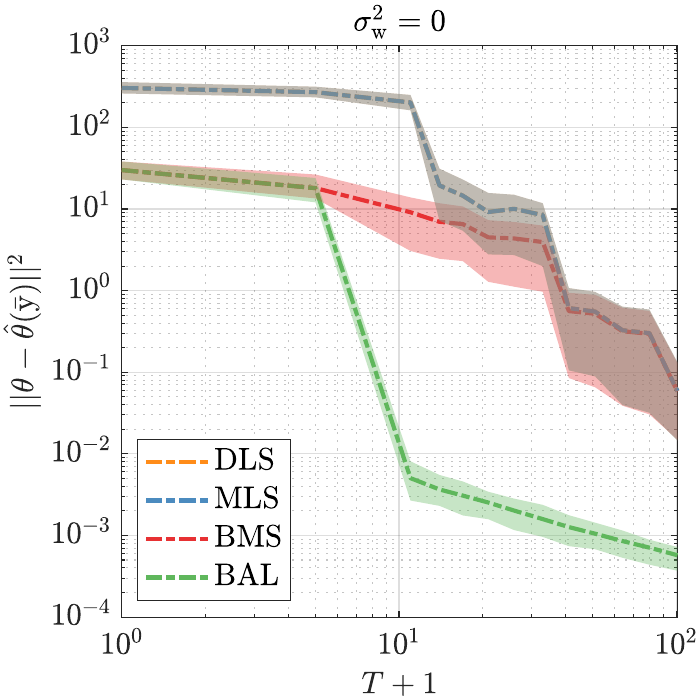} \hfill
\includegraphics[width=0.3285\linewidth]{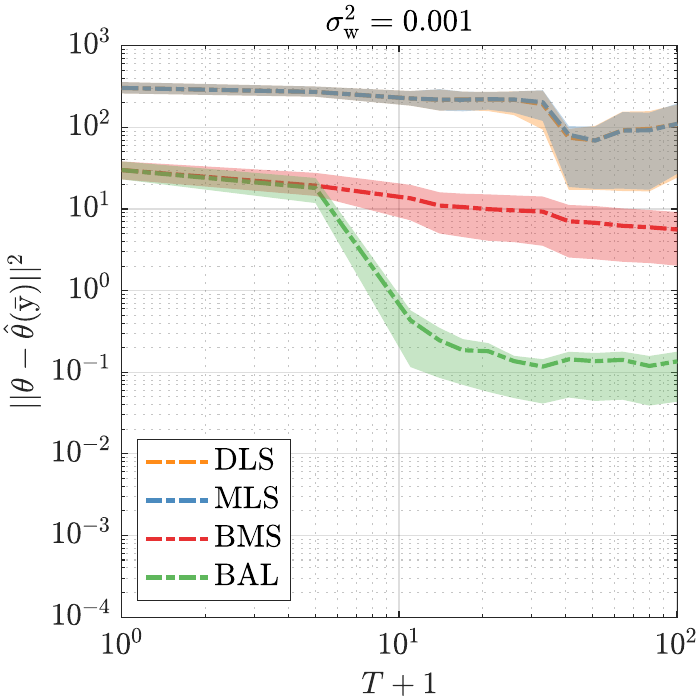} \hfill
\includegraphics[width=0.3285\linewidth]{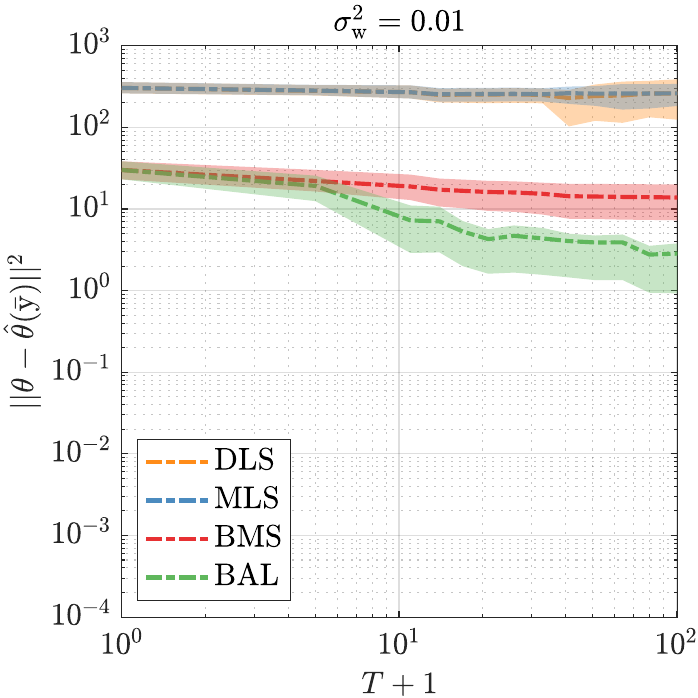}
\caption{Squared errors of DLS, MLS, BMS, and BAL for varying $\Trajectory$}
\label{Fig:MeasurementLen_Error_OptInput}
\end{figure}
\paragraph{\bf Benchmark~4: evaluation with multiple trajectories } 
Lastly, we address the slowing rate of estimation error reduction observed in Figure~\ref{Fig:MeasurementLen_Error_OptInput} of the previous benchmark by employing multiple batches of independent trajectories which demonstrates the consistency condition of the BAL estimator as outlined in Proposition~\ref{proposition:Multi_Batch_Bayesian_MMSE_consistency}. We conduct an experiment using $10,\!000$ simulations, comparing three settings with different numbers of independent trajectories: $\batch = \{1, 11, 101\}$. The same realizations of $\vecWbar$ and $\vecVbar$ were maintained across all cases. When $\batch = 101$, each trajectory consists of $1$ sample; when $\batch = 11$, we have $10$ trajectories of $10$ samples each and $1$ trajectory of $1$ sample; and when $\batch = 1$, there is only $1$ trajectory of $101$ samples, representing the performance of the BAL estimator from the two previous benchmarks with $\Trajectory = 100$. For each setting, optimal inputs $\vecUbar^{\star}$ are obtained using~\eqref{eq:Active_Learning_Optimization}, which includes $\MuX{0}$ for different trajectories except the first. The optimization of $\MuX{0}$ for the first batch is excluded to ensure that the BAL estimates for $\batch = 1$ are identical to those obtained from $\Trajectory = 100$ in Benchmarks~2 and~3. Figure~\ref{Fig:IndepBatch_Error_OptInput} compares the probability density histograms of the squared estimation error for all three settings under process noise variances $\sigma^{2}_{\vecW} = 0.001$ and $0.01$. Unlike Figure~\ref{Fig:ErrorDiff_Histogram}, the horizontal axes here use a logarithmic scale to better resolve differences in the error distributions. The $\sigma^{2}_{\vecW} = 0$ scenario is excluded because, in the absence of process noise, measurements across time steps become statistically independent, providing similar results for all three settings. Results show that the BAL estimator achieves the least estimation error with $101$ independent samples, followed by $11$ independent trajectories, and then $1$ trajectory of $101$ correlated samples. The comparison between the two plots demonstrates that the estimation error increases with process noise, highlighting its impact on estimation performance. This experiment confirms the consistency condition in Proposition~\ref{proposition:Multi_Batch_Bayesian_MMSE_consistency}.
\begin{figure}[htbp]
\centering 
\includegraphics[width=0.4523\linewidth]{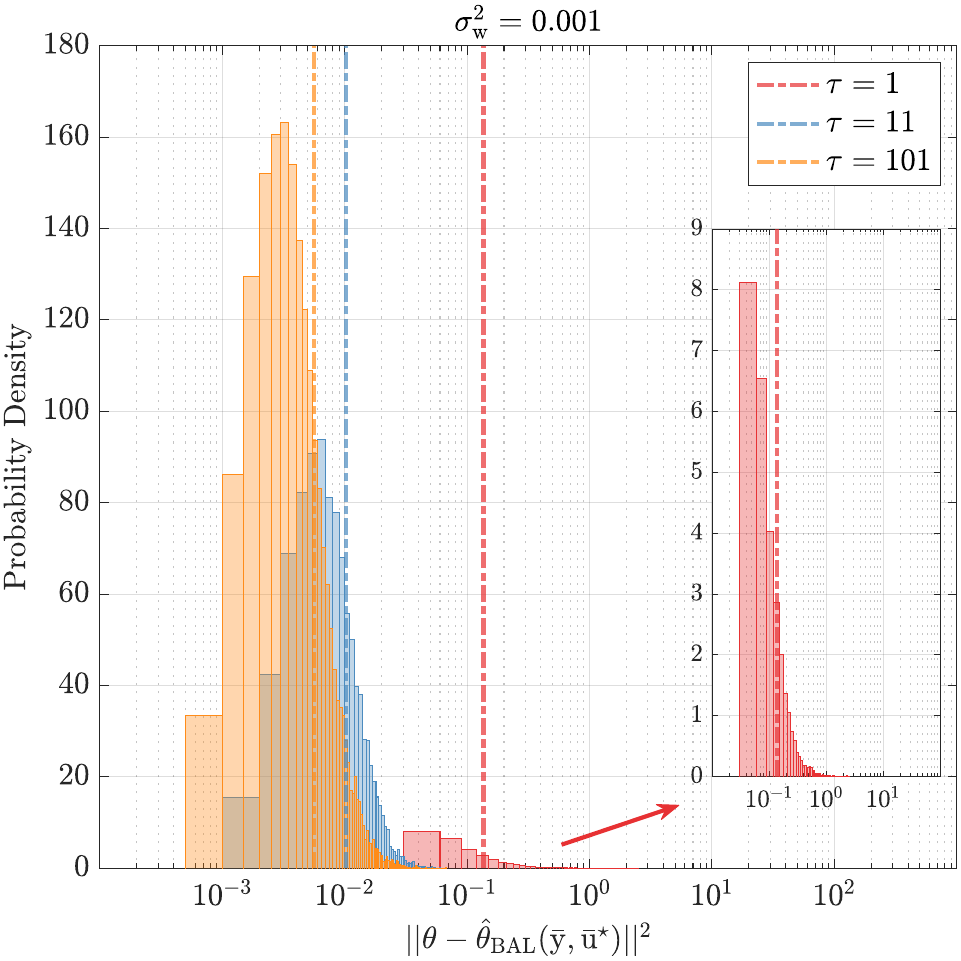} \qquad 
\includegraphics[width=0.4523\linewidth]{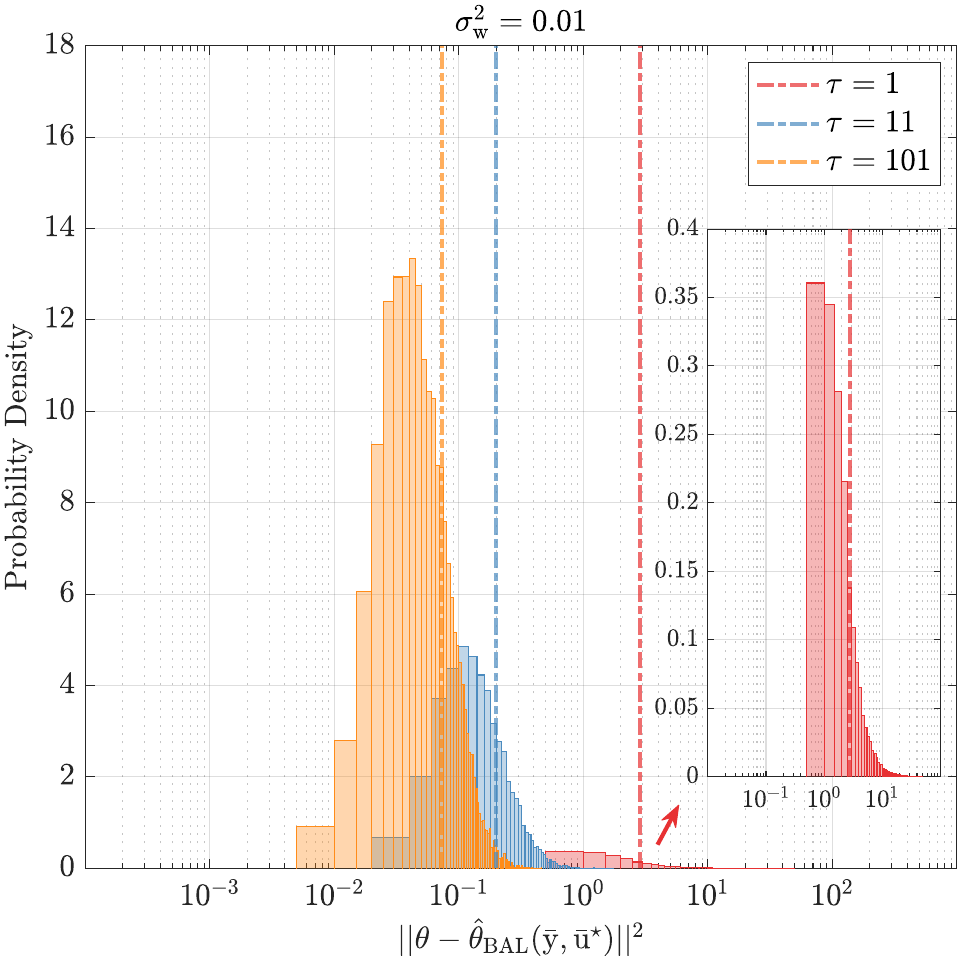}
\caption{Distribution of squared errors for multiple $\tau$}
\label{Fig:IndepBatch_Error_OptInput}
\end{figure}
\paragraph{\bf Experimental analysis summary} 
Benchmark~1 and Benchmark~2 demonstrate the superiority of the Bayesian MMSE affine estimator, showing that it almost surely outperforms the \emph{approximate} regularized least-squares method (cf.\,Figure~\ref{Fig:LSHyperParam_Error} and Figure~\ref{Fig:ErrorDiff_Histogram}). Furthermore, when paired with active learning, the Bayesian MMSE affine estimator achieves the lowest estimation error (cf.\,Figure~\ref{Fig:ErrorDiff_Histogram} and Figure~\ref{Fig:MeasurementLen_Error_OptInput}). Benchmark~3 further investigates the impact of process noise in a marginally stable dynamical system, emphasizing that unbounded growth in process noise slows down estimation error reduction (cf.\,Figure~\ref{Fig:MeasurementLen_Error_OptInput}). As the number of dependent data points increases, the information gained from each new measurement diminishes, preventing convergence to the true value even with infinite measurements. Benchmark~4 highlights the importance of using multiple independent trajectories for accurate parameter estimation under process noise. It demonstrates convergence as the number of independent trajectories increases (cf.\,Figure~\ref{Fig:IndepBatch_Error_OptInput}).


\begin{thebibliography}{10}

\bibitem{barfoot_state_2017}
Timothy~D. Barfoot.
\newblock {\em State Estimation for Robotics}.
\newblock Cambridge University Press, 2017.

\bibitem{bernstein2009matrix}
Dennis~S. Bernstein.
\newblock {\em Matrix Mathematics: Theory, Facts, and Formulas}.
\newblock Princeton University Press, 2009.

\bibitem{bertsekas2015convex}
Dimitri~P. Bertsekas.
\newblock {\em Convex Optimization Algorithms}.
\newblock Athena Scientific, 2015.

\bibitem{bishop_pattern_2006}
Christopher~M. Bishop.
\newblock {\em Pattern Recognition and Machine Learning}.
\newblock Springer, 2006.

\bibitem{bottegai2017identification}
Giulio Bottegai, Ricardo Castro-Garcia, and Johan~A.K. Suykens.
\newblock On the identification of {W}iener systems with polynomial nonlinearity.
\newblock In {\em IEEE Conference on Decision and Control}, pages 6475--6480, 2017.

\bibitem{cambanis1981theory}
Stamatis Cambanis, Steel Huang, and Gordon Simons.
\newblock On the theory of elliptically contoured distributions.
\newblock {\em Journal of Multivariate Analysis}, 11:368--385, 1981.

\bibitem{cedeno2024identification}
Angel~L Cede{\~n}o, Rodrigo~A Gonz{\'a}lez, Rodrigo Carvajal, and Juan~C Ag{\"u}ero.
\newblock Identification of {W}iener state--space models utilizing {G}aussian sum smoothing.
\newblock {\em Automatica}, 166:111707, 2024.

\bibitem{chiuso2019system}
Alessandro Chiuso and Gianluigi Pillonetto.
\newblock System identification: A machine learning perspective.
\newblock {\em Annual Review of Control, Robotics, and Autonomous Systems}, 2:281--304, 2019.

\bibitem{Deistler2001}
Manfred Deistler.
\newblock System identification - general aspects and structure.
\newblock In {\em Model Identification and Adaptive Control: From Windsurfing to Telecommunications}, pages 3--26. Springer, 2001.

\bibitem{fedorov1972theory}
Valerii~V. Fedorov.
\newblock {\em Theory of Optimal Experiments}.
\newblock Academic Press, 1972.

\bibitem{florescu2013handbook}
Ionut Florescu and Ciprian~A. Tudor.
\newblock {\em Handbook of Probability}.
\newblock John Wiley \& Sons, 2013.

\bibitem{gevers2012experiment}
Michel Gevers, Matthias Caenepeel, and Johan Schoukens.
\newblock Experiment design for the identification of a simple {W}iener system.
\newblock In {\em IEEE Conference on Decision and Control}, pages 7333--7338, 2012.

\bibitem{girard2002gaussian}
Agathe Girard, Carl~Edward Rasmussen, Joaquin Qui{\~n}onero~Candela, and Roderick Murray-Smith.
\newblock {G}aussian process priors with uncertain inputs application to multiple-step ahead time series forecasting.
\newblock {\em Advances in Neural Information Processing Systems}, 15:529--536, 2002.

\bibitem{gnedenko2018theory}
Boris~V. Gnedenko.
\newblock {\em Theory of Probability}.
\newblock Routledge, 2018.

\bibitem{hagenblad1999aspects}
Anna Hagenblad.
\newblock {\em Aspects of the Identification of {W}iener Models}.
\newblock {P}h{D} {L}icentiate {T}hesis, Division of Automatic Control, Department of Electrical Engineering, Link\"{o}pings Universitet, 1999.

\bibitem{hagenblad2008maximum}
Anna Hagenblad, Lennart Ljung, and Adrian Wills.
\newblock Maximum likelihood identification of {W}iener models.
\newblock {\em Automatica}, 44:2697--2705, 2008.

\bibitem{hensman2018variational}
James Hensman, Nicolas Durrande, and Arno Solin.
\newblock Variational fourier features for {G}aussian processes.
\newblock {\em Journal of Machine Learning Research}, 18:1--52, 2018.

\bibitem{johnson1987multivariate}
Mark~E. Johnson.
\newblock {\em Multivariate Statistical Simulation: A Guide to Selecting and Generating Continuous Multivariate Distributions}.
\newblock John Wiley \& Sons, 1987.

\bibitem{levy2008principles}
Bernard~C. Levy.
\newblock {\em Principles of Signal Detection and Parameter Estimation}.
\newblock Springer, 2008.

\bibitem{lindqvist2001identification}
Kristian Lindqvist and H{\aa}kan Hjalmarsson.
\newblock Identification for control: Adaptive input design using convex optimization.
\newblock In {\em IEEE Conference on Decision and Control}, pages 4326--4331, 2001.

\bibitem{lindsten2013bayesian}
Fredrik Lindsten, Thomas~B. Sch{\"o}n, and Michael~I. Jordan.
\newblock Bayesian semiparametric {W}iener system identification.
\newblock {\em Automatica}, 49:2053--2063, 2013.

\bibitem{liu2021random}
Fanghui Liu, Xiaolin Huang, Yudong Chen, and Johan~A.K. Suykens.
\newblock Random features for kernel approximation: A survey on algorithms, theory, and beyond.
\newblock {\em IEEE Transactions on Pattern Analysis and Machine Intelligence}, 44:7128--7148, 2021.

\bibitem{Ljung1999SysId}
Lennart Ljung.
\newblock {\em System Identification (second edition): Theory for the User}.
\newblock Prentice Hall PTR, 1999.

\bibitem{ljung2010perspectives}
Lennart Ljung.
\newblock Perspectives on system identification.
\newblock {\em Annual Reviews in Control}, 34:1--12, 2010.

\bibitem{malitsky20adaptive}
Yura Malitsky and Konstantin Mishchenko.
\newblock Adaptive gradient descent without descent.
\newblock In {\em International Conference on Machine Learning}, pages 6702--6712, 2020.

\bibitem{mania2022active}
Horia Mania, Michael~I. Jordan, and Benjamin Recht.
\newblock Active learning for nonlinear system identification with guarantees.
\newblock {\em Journal of Machine Learning Research}, 23:1--30, 2022.

\bibitem{mchutchon2011gaussian}
Andrew McHutchon and Carl~Edward Rasmussen.
\newblock {G}aussian process training with input noise.
\newblock {\em Advances in Neural Information Processing Systems}, 24:182--190, 2011.

\bibitem{mehra1976synthesis}
Raman~K. Mehra.
\newblock Synthesis of optimal inputs for multiinput-multioutput ({M}imo) systems with process noise part i: Frequenc y-domain synthesis part ii: Time-domain synthesis.
\newblock In {\em Mathematics in Science and Engineering}, pages 211--249. Elsevier, 1976.

\bibitem{Nelles2020}
Oliver Nelles.
\newblock {\em Nonlinear System Identification: From Classical Approaches to Neural Networks, Fuzzy Models, and {G}aussian Processes}.
\newblock Springer International Publishing, 2020.

\bibitem{paduart2010identification}
Johan Paduart, Lieve Lauwers, Jan Swevers, Kris Smolders, Johan Schoukens, and Rik Pintelon.
\newblock Identification of nonlinear systems using polynomial nonlinear state space models.
\newblock {\em Automatica}, 46:647--656, 2010.

\bibitem{pillonetto2022regularized}
Gianluigi Pillonetto, Tianshi Chen, Alessandro Chiuso, Giuseppe De~Nicolao, and Lennart Ljung.
\newblock {\em Regularized System Identification: Learning Dynamic Models from Data}.
\newblock Springer Nature, 2022.

\bibitem{pillonetto2010new}
Gianluigi Pillonetto and Giuseppe De~Nicolao.
\newblock A new kernel-based approach for linear system identification.
\newblock {\em Automatica}, 46:81--93, 2010.

\bibitem{pillonetto2014kernel}
Gianluigi Pillonetto, Francesco Dinuzzo, Tianshi Chen, Giuseppe De~Nicolao, and Lennart Ljung.
\newblock Kernel methods in system identification, machine learning and function estimation: A survey.
\newblock {\em Automatica}, 50:657--682, 2014.

\bibitem{rahimi2007random}
Ali Rahimi and Benjamin Recht.
\newblock Random features for large-scale kernel machines.
\newblock {\em Advances in Neural Information Processing Systems}, 20:1177--1184, 2007.

\bibitem{rasmussen_gaussian_2006}
Carl~Edward Rasmussen and Christopher K.~I. Williams.
\newblock {\em {G}aussian Processes for Machine Learning}.
\newblock {MIT} Press Cambridge, 2006.

\bibitem{risuleo2020nonparametric}
Riccardo~S Risuleo and H{\aa}kan Hjalmarsson.
\newblock Nonparametric models for {H}ammerstein-{W}iener and {W}iener-{H}ammerstein system identification.
\newblock {\em IFAC-PapersOnLine}, 53:400--405, 2020.

\bibitem{rudin_fourier_1962}
Walter Rudin.
\newblock {\em Fourier Analysis on Groups}.
\newblock Interscience Publishers, 1962.

\bibitem{schoukens2019nonlinear}
Johan Schoukens and Lennart Ljung.
\newblock Nonlinear system identification: A user-oriented road map.
\newblock {\em IEEE Control Systems Magazine}, 39:28--99, 2019.

\bibitem{schoukens2017identification}
Maarten Schoukens and Koen Tiels.
\newblock Identification of block-oriented nonlinear systems starting from linear approximations: A survey.
\newblock {\em Automatica}, 85:272--292, 2017.

\bibitem{settles2012active}
Burr Settles.
\newblock {\em Active Learning}.
\newblock Morgan \& Claypool Publishers, 2012.

\bibitem{soderstrom_system_1989}
Torsten Söderström and Petre Stoica.
\newblock {\em System Identification}.
\newblock Prentice Hall, 1989.

\bibitem{totterman2009support}
Stefan T{\"o}tterman and Hannu~T. Toivonen.
\newblock Support vector method for identification of {W}iener models.
\newblock {\em Journal of Process Control}, 19:1174--1181, 2009.

\bibitem{valenzuela2017robust}
Patricio~E. Valenzuela, Johan Dahlin, Cristian~R. Rojas, and Thomas~B. Sch{\"o}n.
\newblock On robust input design for nonlinear dynamical models.
\newblock {\em Automatica}, 77:268--278, 2017.

\bibitem{valenzuela2013optimal}
Patricio~E. Valenzuela, Cristian~R. Rojas, and H{\aa}kan Hjalmarsson.
\newblock Optimal input design for non-linear dynamic systems: a graph theory approach.
\newblock In {\em IEEE Conference on Decision and Control}, pages 5740--5745, 2013.

\bibitem{wagenmaker2020active}
Andrew Wagenmaker and Kevin Jamieson.
\newblock Active learning for identification of linear dynamical systems.
\newblock {\em Conference on Learning Theory}, 125:3487--3582, 2020.

\end{thebibliography}
\end{document}